\documentclass{article}

 \usepackage[nonatbib,final]{neurips_2022}

\usepackage[utf8]{inputenc} %
\usepackage[T1]{fontenc}    %
\usepackage[hidelinks,colorlinks,linkcolor=black,citecolor=orange]{hyperref}   %
\usepackage{url}            %
\usepackage{booktabs}       %
\usepackage{amsfonts}       %
\usepackage{nicefrac}       %
\usepackage{microtype}      %
\usepackage{xcolor}         %
\usepackage[citestyle=numeric,natbib=true,bibstyle=numeric,maxbibnames=99]{biblatex}
\usepackage{graphicx}
\usepackage{caption}
\usepackage{subcaption}
\usepackage{amsmath,amssymb,amsthm}
\usepackage{bbm}
\usepackage{wrapfig}
\usepackage{enumitem}

\newcommand\E{{\mathbb{E}}}   %

\renewcommand\l{{\mathcal{L}}}
\renewcommand\E{{\mathcal{E}}}
\newcommand\V{{\mathcal{V}}}

\newcommand\F{{\mathcal{F}}}
\renewcommand\S{{\mathcal{S}}}

\newcommand\va{{\boldsymbol{a}}}

\newcommand\vb{{\boldsymbol{b}}}

\newcommand\vx{{\boldsymbol{x}}}

\newcommand\vz{{\boldsymbol{z}}}

\newcommand{\ceil}[1]{\left\lceil #1 \right\rceil}

\renewcommand\P{{\mathcal{P}}}
\newcommand\res{{\mathrm{res}}}
\newcommand\cnn{{\mathrm{cnn}}}

\def\rva{{\mathbf{a}}}
\def\rvb{{\mathbf{b}}}

\def\rvu{{\mathbf{i}}}

\def\rvs{{\mathbf{s}}}

\def\rvu{{\mathbf{u}}}
\def\rvv{{\mathbf{v}}}
\def\rvw{{\mathbf{w}}}
\def\rvx{{\mathbf{x}}}

\def\rmH{{\mathbf{H}}}

\def\rmW{{\mathbf{W}}}

\DeclareMathOperator{\In}{\mathsf{in}}
\DeclareMathOperator{\Out}{\mathsf{out}}
\DeclareMathOperator{\id}{\mathsf{id}}

\theoremstyle{plain}
\newtheorem{theorem}{Theorem}[section]
\newtheorem{proposition}[theorem]{Proposition}
\newtheorem{lemma}[theorem]{Lemma}
\newtheorem{corollary}[theorem]{Corollary}
\theoremstyle{definition}
\newtheorem{definition}[theorem]{Definition}
\newtheorem{assumption}[theorem]{Assumption}
\theoremstyle{remark}
\newtheorem{remark}[theorem]{Remark}

\title{Transition to Linearity of  General Neural Networks with  Directed Acyclic Graph Architecture}

\author{%
  Libin Zhu\thanks{ Department of Computer Science \& Halicioğlu Data Science Institute, University of California, San Diego. E-mail: \texttt{l5zhu@ucsd.edu}}~~~~~~~~
   Chaoyue Liu\thanks{Halicioğlu Data Science Institute, University of California, San Diego. E-mail: \texttt{chl212@ucsd.edu}}~~~~~~~~
   Mikhail Belkin\thanks{Halicioğlu Data Science Institute \& Department of Computer Science, University of California, San Diego. E-mail: \texttt{mbelkin@ucsd.edu}}
  \\
}

\begin{document}

\maketitle

\begin{abstract}

In this paper we show that feedforward neural networks corresponding to arbitrary  directed acyclic graphs undergo transition to linearity as their ``width'' approaches infinity. The width of these general networks is characterized by the minimum in-degree of their neurons, except for the input and first layers. Our results identify the mathematical structure underlying transition to linearity and generalize a number of recent works aimed at characterizing transition to linearity or constancy of the Neural Tangent Kernel for standard architectures.

\end{abstract}

\section{Introduction}

A remarkable property of wide neural networks, first discovered in~\cite{jacot2018neural} in terms of the constancy of the Neural Tangent Kernel along the optimization path, is that they transition to linearity (using the terminology from~\cite{liu2020linearity}), i.e., are approximately linear in a ball of a fixed radius. 
There has been an extensive study of this phenomenon for different types of standard neural networks architectures including fully-connected neural networks (FCNs), convolutional neural networks (CNNs), ResNets ~\cite{lee2019wide,chizat2019lazy,arora2019exact,hanin2019finite}. 
Yet the scope of the transition to linearity and the underlying mathematical structure  has not been made completely clear.

In this paper, we show that the property of transition to linearity  holds for a  much broader class of neural networks -- {\it feedforward neural networks}. The architecture of a feedforward neural network can generically be described by a DAG~\cite{you2020graph,wortsman2019discovering,mcclelland1987parallel}: the vertices and the edges correspond to the neurons and the trainable weight parameters of a neural network, respectively.
This DAG structure includes  standard network architectures e.g., FCNs, CNNs, ResNets, as well as DenseNets~\cite{huang2017densely}, whose property of transition to linearity has not been studied in literature.
This generalization shows that the transition to linearity, or the constant Neural Tangent Kernel, does not depend on the specific designs of the networks, and is a more fundamental and universal property.

We define the width of a feedforward neural network as the minimum in-degree of all neurons except for the input and first layers, which 
is a natural generalization of the 
 the minimum number of neurons in hidden layers which is how the width is defined for standard architectures.
For a feedforward neural network, we show it  transitions to linearity if its width goes to infinity as long as the in-degrees of individual neurons are bounded by a polynomial of the network width.
Specifically, we control the deviation of the network  function from its linear approximation by the spectral norm of the Hessian of the network function, which, as we show  vanishes in a ball of fixed radius, in the infinite width limit.  Interestingly, we observe that not only the output neurons, but any pre-activated neuron in the hidden layers of a feedforward neural network can be regarded as a function with respect to its parameters, which will also transition to linearity as the width  goes to infinity.

The key technical difficulty is that all existing analyses for transition to linearity or constant NTK do not apply to this general DAG setting. Specifically, 
those analyses assume in-degrees of neurons are either the same or proportional to each other up to a constant ratio \cite{du2018gradientdeep,lee2019wide,arora2019exact,zou2020gradient,liu2020linearity,allen2019convergence}. However, the general DAG setting allows different scales of neuron in-degrees, for example, 
the largest in-degree can be polynomially large in the smallest in-degree.
In such scenarios, the $(2,2,1)$-norm in \cite{liu2020linearity} and the norm of  parameter change in~\cite{du2018gradientdeep,lee2019wide}  scales with the maximum of in-degrees which causes a trivial bound on the NTK change.
Instead, we introduce a different set of tools based on the tail bound for the norm of matrix Gaussian series~\cite{tropp2015introduction}. Specifically, we show that the Hessian of the network function takes the form of matrix Gaussian series, whose matrix variance relies on the Hessian of connected neurons. Therefore, we reconcile the in-degree difference by building a recursive relation between the Hessian of neurons,  which exactly cancels out the in-degree with the scaling factor.

Transition to linearity helps understand the training dynamics of wide neural networks and plays an important role in developing the optimization theory for them, as has been shown for certain particular wide neural networks~\cite{du2018gradient,du2018gradientdeep,chizat2019lazy,lee2019wide,zou2019improved,zou2020gradient}.
While transition to linearity is not a necessary condition for successful optimization, it provides a powerful tool for analyzing optimization for many different architectures. 
Specifically, transition to linearity in a ball of sufficient radius combined with a lower bound on the norm of the gradient at its center is sufficient to demonstrate the PL$^*$ condition~\cite{liu2020loss} (a version of the Polyak-{\L}ojasiewicz condition~\cite{polyak1963gradient,lojasiewicz1963topological}) which ensures convergence of optimization. We discuss this connection and provide one such lower bound in Section~\ref{sec:lowerbound}.

\paragraph{Summary of contributions.} 

 We show the phenomenon of transition to linearity in general feedforward neural networks corresponding to a DAG with large in-degree. Specifically, under the assumption that the maximum in-degree of its neurons is bounded by a polynomial of the width $m$ (the minimum in-degree), we prove that the spectral norm of the Hessian of   a feedforward neural network is bounded by $\tilde{O}(1/\sqrt{m})$ in an $O(1)$ ball. 
 Our results generalize the existing literature on the linearity of wide feedforward neural networks.
    We discuss connections to optimization. Under additional assumptions we show that the norm of the gradient of a feedforward neural network is bounded away from zero at initialization. Together with the Hessian bound this implies convergence of gradient descent for the loss function.

\subsection{Notations}
We use bold lowercase letters, e.g., $\rvw$, to denote vectors, capital letters, e.g., $A$, to denote matrices, and bold capital letters, e.g., $\rmH$, to denote higher order tensors or matrix tuples. For a matrix $A$, we use $A_{[i,:]}$ to denote its $i$-th row and $A_{[:,i]}$ to denote its $j$-th column.

We use $\nabla_\rvw f(\rvw_0)$ to denote the gradient of $f$ with respect to $\rvw$ at $\rvw_0$, and $H_f(\rvw)$ to denote Hessian matrix (second derivative) of $f$ with respect to $\rvw$. 
For vectors, we use $\|\cdot\|$ to denote Euclidean norm. For matrices, we use $\|\cdot\|$ to denote spectral norm and $\|\cdot\|_F$ to denote Frobenius norm. We use $\|\cdot\|_\infty$ to denote function $L_\infty$ norm. For a set $\mathcal{S}$, we use $|\mathcal{S}|$ to denote the cardinality of the set. For $n>0$, $[n]$ denotes the set $\{1,2,...,n\}$.

We use big-$O$ notation to hide constant factors, and use big-$\tilde{O}$ notation to additionally hide logarithmic factors.
In this paper, the argument of $O$/$\tilde{O}(\cdot)$ is always with respect to the network width.

Given a vector $\rvw$ and a constant $R>0$, we define a Euclidean ball $\mathsf{B}(\rvw,R)$ as:
\begin{align}
    \mathsf{B}(\rvw,R) := \left\{\rvv:\|\rvv - \rvw\| \leq R\right\}.
\end{align}

\section{Neural networks with  acyclic graph architecture}\label{sec:general_nn}
In this section, we provide a definition and notation for general feedforward neural networks with an arbitrary DAG structure.  This definition includes standard feedforward neural network architectures, such as FCNs and DenseNet.

\subsection{Defining feedforward neural networks}

 \begin{figure}[ht]
    \centering
      \begin{subfigure}[t]{0.42\textwidth}
        \includegraphics[width=\linewidth]{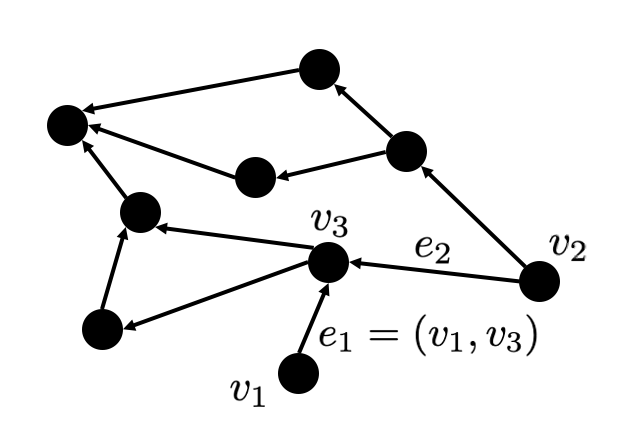} 
         \caption{}
     \end{subfigure}
     \begin{subfigure}[t]{0.45\textwidth}
        \includegraphics[width=\linewidth]{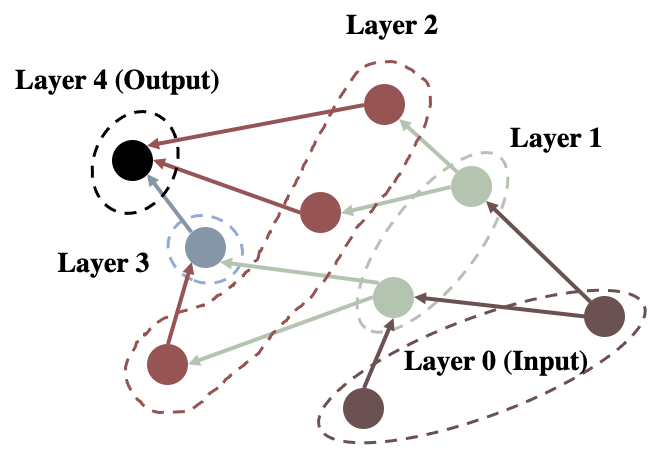}
         \caption{}
     \end{subfigure}
\caption{{\bf (a):} {\bf An example of  directed acyclic graph.} $v_1$, $v_2$ and $v_3$ are three vertices and $e_1$, $e_2$ are two edges of the graph. $v_3$ has two incoming edges $e_1$ and $e_2$ which connects to $v_1$ and $v_2$ respectively. 
{\bf (b):} {\bf Organizing the vertices into layers.} The vertices with $0$ in-degree are in $0$-th layer (or input layer), and last layer are called output layer in which the vertices have $0$ out-degree. Note that the layer index is determined by the longest path from the inputs $\mathcal{V}_{\mathrm{input}}$, for example, the neuron in layer $3$.
} \label{fig:label_graph}
\vspace{-10pt}
\end{figure}
{\bf Graph Structure.} Consider a directed acyclic graph (DAG) $\mathcal{G} = (\mathcal{V}, \mathcal{E})$, where $\mathcal{V}$ and $\mathcal{E}$ denote the sets of vertices and edges, respectively. See the left panel of Figure \ref{fig:label_graph}, for an illustrative example. For a directed edge $e\in\mathcal{E}$, we may also use the notation $e= (v_1,v_2)$ to explicitly write out the start vertex $v_1$ and end vertex $v_2$. 

For a vertex $v\in\mathcal{V}$, we denote its in-degree, $\In(v)$, by the number of incoming edges (edges that end with it):
\begin{equation}
   {\mathsf{in}}(v) = |\mathcal{S}_{{\mathrm{in}}}(v)|, \ \   \mathrm{with} \ \mathcal{S}_{{\mathrm{in}}}(v) :=\{u\in\mathcal{V} : (u,v)\in\mathcal{E}\}. \nonumber
\end{equation} 
Similarly, for a vertex $v\in\mathcal{V}$, we denote its out-degree $\Out(v)$  by the number of outgoing edges (edges that start from it): 
\begin{equation}
  {\Out}(v) = |\mathcal{S}_{\mathrm{out}}(v)|, \ \   \mathrm{with} \ \mathcal{S}_{{\mathrm{out}}}(v) :=\{u\in\mathcal{V} : (v,u)\in\mathcal{E}\}. \nonumber
\end{equation} 

We call the set of vertices with zero in-degrees  {\it input}: $\mathcal{V}_{\mathrm{input}} = \{v\in\mathcal{V}: {\In}(v)=0\}$, and the set of vertices with zero out-degrees  {\it output} $\mathcal{V}_{\mathrm{output}} = \{v\in\mathcal{V}: {\Out}(v)=0\}$.

\begin{definition}
For each vertex $v\in \mathcal{V}\backslash\mathcal{V}_{\mathrm{input}}$, its distance $p(v)$, to the input $\V_{\mathrm{input}}$,  is defined to be the maximum length of all paths that start from a vertex within $\mathcal{V}_{\mathrm{input}}$ and end with $v$. 
\end{definition}
It is easy to check that $p(v) = 0$ if $v \in\mathcal{V}_{\mathrm{input}}$.

{\bf Feedforward neural network.} Based on a given DAG architecture, we define the feedforward neural network.  Each individual vertex corresponds to a neuron additionally  equipped  with a scalar function (also called activation function).  Each edge is associated with a real-valued weight, a trainable parameter. Each neuron is defined as a function of the weight parameters and the adjacent neurons connected by its incoming edges. The feedforward neural network is considered as the output neurons, corresponding to the output $\V_{\mathrm{output}}$, of all  weight parameters and input neurons which correspond to the input $\V_{\mathrm{input}}$.
Formally, we define the feedforward neural network as follows.
\begin{definition}[Feedforward neural network]\label{def:fnn}
Consider a DAG $\mathcal{G} = (\V,\E)$. For each vertex $v\in\mathcal{V}\backslash\mathcal{V}_{\mathrm{input}}$, we associate it with an activation function $\sigma_v(\cdot): \mathbb{R} \to \mathbb{R}$ and each of its incoming edges $e=(u,v)\in\mathcal{E}$ with a weight variable $w_e =  w_{(u,v)}$. Then we define the following functions:
\begin{equation}\label{eq:local_function}
    f_v =\sigma_v(\tilde{f}_v), \quad \tilde{f}_v= \frac{1}{\sqrt{\In(v)}}\sum_{u\in\mathcal{S}_{\mathrm{in}}(v)} w_{(u,v)}f_u.
\end{equation}
When $v\in \V_{\mathrm{input}}$, $f_v$ is prefixed as the input data, and we denote $f_{\mathrm{input}}:=\{f_v: v\in \V_{\mathrm{input}}\}$.
 For $v\notin \V_{\mathrm{input}}$, we call $f_v$  {\it neurons} and $\tilde{f}_v$  {\it pre-activations}.
With necessary composition of functions, each $f_v$, and $\tilde{f}_v$, can be regarded as a function of all related weight variables and inputs $f_{\mathrm{input}}$.
The {\it feedforward neural network} is defined to be the function corresponding to the output $\V_{\mathrm{output}}$:
\begin{equation}\label{eq:f_output}
    f(\mathcal{W};f_{\mathrm{input}}) := f_{\mathrm{output}}=\{f_v: v\in \V_{\mathrm{output}}\},
\end{equation}
where $\mathcal{W}:=\{w_e: e\in\mathcal{E}\}$ denotes the set of all the weight variables.
\end{definition}
\begin{remark}
The validity of the definition is guaranteed by the fact that the DAG is acyclic. It makes sure that the dependence of each function $f_v$ on other neurons can pass all the way down to the input $f_{\mathrm{input}}$, through Eq.~(\ref{eq:local_function}).
\end{remark}
\begin{remark}
For $v\in\V_{\mathrm{input}}\bigcup\V_{\mathrm{output}}$, we use the identity function $\mathbb{I}(\cdot)$ as the activation functions.
\end{remark}

\paragraph{Weight initialization and inputs.} Each weight parameter $w_e \in \mathcal{W}$ is initialized i.i.d. following the standard normal distribution i.e., $\mathcal{N}(0,1)$. The inputs are considered given, usually determined by datasets. Under this initialization, we introduce the scaling factor $1/\sqrt{\In(v)}$ in Eq.~(\ref{eq:basic_op}) to control the value of neurons to be of order $O(1)$. Note that this initialization is an extension of the NTK initialization~\cite{jacot2018neural}, which was defined for FCNs therein. 
\paragraph{Generality of DAG architecture.} 
Including FCNs and DenseNets~\cite{huang2017densely} as special examples, the class of feedforward neural networks allows much more choices of architectures, for example, neural networks with randomly dropped edges.
Please see detailed discussions about these specific examples in Appendix~\ref{app:example}.
We note that our definition of feedforward neural networks does not directly include networks with nont-trainable skip connections, e.g., ResNets, and networks with shared weights, e.g., CNNs. However, with a slight modification of the analysis,  the property of transition to linearity still holds. See the detailed discussion in Appendix~\ref{sec:skip_connection} and~\ref{sec:cnn}.

\subsection{Organizing feedforward networks into layers }\label{subsec:nn_notation}
The architecture of the feedforward neural network is determined by the DAG $\mathcal{G}$. The complex structures of DAGs often lead to complicated neural networks, which are hard to analyze. 

For the ease of analysis, we organize the neurons of the feedforward neural network into {\it layers}, which are sets of neurons. 

\begin{definition}[Layers]\label{def:layer}
Consider a feedforward neural network $f$ and its corresponding graph structure $\mathcal{G}$. A layer of the network is defined to be the set of neurons which have the same  distance $p$ to the inputs. Specifically, the $\ell$-th layer, denoted by $f^{(\ell)}$, is
\begin{equation}\label{eq:layer}
    f^{(\ell)} = \{f_v: p(v) = \ell, v\in \mathcal{V}, \ell\in \mathbb{N}\}.
\end{equation}
\end{definition}

It is easy to see that the layers are mutually exclusive, and the layer index $\ell$ is labeled from $0$ to $\ell$, where $L+1$ is the total number of layers in the network.
As $p(v) = 0$ if and only if $v\in \mathcal{V}_{\mathrm{input}}$,  the $0$-th layer $f^{(0)}$ is exactly the input layer $f_{\mathrm{input}}$. 
The right panel of Figure~\ref{fig:label_graph} provides an illustrative example of the layer structures.

In general, the output neurons $f_{\mathrm{output}}$ (defined in Eq.~(\ref{eq:f_output})) do not have to be in the same layer. For the convenience of presentation and analysis, we assume that all the output neurons are in the last layer, i.e., layer $\ell$, which is the case for most of commonly used neural networks, e.g., FCNs and CNNs.
Indeed, our analysis applies to every output neuron (see Theorem~\ref{cor:bound_each_hessian}), even if they are not in the same layer.

With the notion of network layers, we rewrite the neuron functions Eq.~(\ref{eq:local_function}), as well as related notations, to reflect the layer information. 

For $\ell$-layer, $\ell=0,1,\cdots,L$, we denote the total number of neurons as $d_\ell$, and rewrite the layer function $f^{(\ell)}$  into a form of vector-valued function
\begin{align*}
    f^{(\ell)} = \left(f_1^{(\ell)}, f_2^{(\ell)},...,f_{d_\ell}^{(\ell)}\right)^T,
\end{align*}
where we use $f_i^{(\ell)}$ with index $i=1,2,\cdots,d_\ell$ to denote each individual neuron. Correspondingly, we denote its vertex as $v_i^{(\ell)}$, and  $\mathcal{S}_{i}^{(\ell)}:=\mathcal{S}_{\mathrm{in}}(v_i^{(\ell)})$. Hence, the in-degree $\In(v_i^{(\ell)})$, denoted as $m_i^{(\ell)}$ here, is equivalent to the cardinality of the set $\mathcal{S}_{i}^{(\ell)}$. 
\begin{remark}
Note that $m_i^{(\ell)}$, with the superscript $\ell$, denotes an in-degree, i.e., the number of neurons that serve as direct inputs to the current neuron in $\ell$-th layer. In the context of FCNs, $m_i^{(\ell)}$ is equivalent to the size of its previous layer, i.e., $(\ell-1)$-th layer, and is often denoted as $m^{(\ell-1)}$ in literature.
\end{remark}

To write the summation in Eq.~(\ref{eq:local_function}) as a matrix multiplication, we further introduce the following two vectors: (a), $f_{\mathcal{S}_i^{(\ell)}}$ represents the vector that consists of neuron components $f_v$ with $v\in\mathcal{S}_i^{(\ell)}$; (b), $\rvw_{i}^{(\ell)}$ represents the  vector that consists of weight parameters $w_{(u,v_i^{(\ell)})}$ with $u \in\mathcal{S}_i^{(\ell)}$. Note that both vectors $f_{\mathcal{S}_i^{(\ell)}}$ and $\rvw_{i}^{(\ell)}$ have the same dimension $m_i^{(\ell)}$.

With the above notation, the neuron functions Eq.~(\ref{eq:local_function}) can be equivalently rewritten as:
\begin{align}\label{eq:basic_op}
    f_i^{(\ell)} =\sigma_{i}^{(\ell)}(\tilde{f}_i^{(\ell)}), \ \  \tilde{f}_i^{(\ell)}= \frac{1}{\sqrt{m_{i}^{(\ell)}}}\left(\rvw_i^{(\ell)}\right)^T f_{\S^{(\ell)}_i}.
\end{align}
For any $\ell\in[L]$, we denote the weight parameters corresponding to all incoming edges toward neurons at layer $\ell$ by
\vspace{-7pt}
\begin{align}\label{eq:def_weights}
    \rvw^{(\ell)} := \left( (\rvw_1^{(\ell)})^T,...,(\rvw_{d_{\ell}}^{(\ell)})^T\right)^T~~~ \ell\in[L].
\end{align}
Through the way we define the feedforward neural network, the output of the neural network is a function of all the weight parameters and the input data, hence we denote it by 
\begin{align}\label{eq:output_f}
    f(\rvw;\vx) := f^{(L)} = \left(f_1^{(L)},...,f_{d_L}^{(L)}\right)^T,
\end{align}
where $\rvw$ is the collection of all the weight parameters, i.e., $ \rvw := \left((\rvw^{(1)})^T,...,(\rvw^{(L)})^T\right)^T \in \mathbb{R}^{\sum_\ell\sum_i m_i^{(\ell)}}$.

With all the notations, for a feedforward neural network, we formally define the width of it:
\begin{definition}[Network width]\label{def:width}
The width $m$ of a feedforward neural network is the minimum in-degree of all the neurons except those in the input and first layers:
\begin{align}
    m:= \inf_{\ell\in\{2,...,L\},i\in [d_\ell]}m_i^{(\ell)}.
\end{align} 
\end{definition}
\begin{remark}
Note that, the network width $m$ is determined by the in-degrees of neurons except for the input and first layers, and not necessarily relates the number of neurons in hidden layers. But for certain architectures e.g., FCNs, these two coincide that the minimum in-degree after the first layer is the same as the minimum hidden layer size.
\end{remark}

We say a feedforward neural network is {\it wide} if its width $m$ is large enough. In this paper, we consider wide feedforward neural networks with a fixed number of layers.

\section{Transition to linearity of feedforward neural networks}\label{sec:control_hessian}
In this section, we show that the feedforward neural networks exhibit the phenomenon of transition to linearity, which was previously observed in specific types of neural networks.  

Specifically, we prove that a feedforward neural network $f(\rvw;\vx)$, when considered as a function of its weight parameters $\rvw$, is arbitrarily close to a {\it linear} function in the ball $\mathsf{B}(\rvw_0,R)$ given constant $R>0$, where $\rvw_0$ is randomly initialized, as long as the width of the network is sufficiently large.

First, we make the following assumptions on the input $\vx$ and the activation functions:
\begin{assumption}\label{assump:input}
The input is uniformly upper bounded, i.e., $\|\vx\|_\infty \leq C_\vx$ for some constant $C_\vx >0$.
\end{assumption}

\begin{assumption}\label{assump:sigma}
All the activation functions $\sigma(\cdot)$ are twice differentiable, and there exist constants $\gamma_0,\gamma_1,\gamma_2 >0$ such that, for all activation functions, $|\sigma(0)| \leq \gamma_0$ and the following Lipschitz continuity and smoothness conditions are satisfied
\begin{align*}
    &\left|\sigma'(z_1) - \sigma'(z_2)\right| \leq \gamma_1|z_1-z_2|,\\
        &\left|\sigma''(z_1) - \sigma''(z_2)\right| \leq \gamma_2|z_1-z_2|,~~\forall z_1,z_2\in\mathbb{R}.
\end{align*}

\end{assumption}

We note that the above two assumptions are very common in literature. Although ReLU does not satisfy Assumption~\ref{assump:sigma} due to non-differentiability at point $0$, we 
believe our main claims still hold as ReLU can be approximated arbitrarily closely by some differentiable function which satisfies our assumption.

\begin{remark}
By assuming all the activation functions are twice differentiable, it is not hard to see that the feedforward neural network i.e., Eq.~(\ref{eq:output_f}) is also twice differentiable.
\end{remark}
{\bf Taylor expansion.} To study the linearity of a general feedforward neural network, we consider its Taylor expansion with second order Lagrange remainder term.  Given a point $\rvw_0$, we can write the network function $f(\rvw)$ (omitting the input argument for simplicity) as 
\begin{align}\label{eq:taylor_exp}
    f(\rvw) &= \underbrace{f(\rvw_0) +  (\rvw-\rvw_0)^T \nabla_\rvw f(\rvw_0)}_{f_{\mathrm{lin}}(\rvw)}+ \underbrace{\frac{1}{2}(\rvw-\rvw_0)^T H_f(\xi) (\rvw-\rvw_0)}_{\mathcal{R}(\rvw)},
\end{align}
where $\xi$ is a point on the line segment between $\rvw_0$ and $\rvw$. Above, $f_{\mathrm{lin}}(\rvw)$ is a linear function and $\mathcal{R}(\rvw)$ is the Lagrange remainder term. 

In the rest of the section, we will show that in a ball $\mathsf{B}(\rvw_0,R)$ of any constant radius $R>0$,
\begin{align}\label{eq:lag_re}
    |\mathcal{R}(\rvw)| = \tilde{O}\left({1}/{\sqrt{m}}\right)
\end{align}
where $m$ is the network width (see Definition~\ref{def:width}). Hence, $f(\rvw)$ can be arbitrarily close to its linear approximation $f_{\mathrm{lin}}(\rvw)$ with sufficiently large $m$. Note that in Eq.~(\ref{eq:taylor_exp}), we consider a single output of the network function. The same analysis can be applied to multiple outputs (see Corollary~\ref{cor:bound_multi_hessian}).
\begin{remark}
For a general function, the remainder term $\mathcal{R}(\rvw)$ is not expected to vanish at a finite distance from $\rvw_0$. Hence, the transition to linearity in the ball $\mathsf{B}(\rvw_0,R)$ is a non-trivial property. 
On the other hand,  the radius $R$ can be set to be large enough to contain the whole optimization path of GD/SGD for various types of wide neural networks (see~\cite{liu2020loss,zou2019improved}, also indicated in ~\cite{du2018gradient,du2018gradientdeep,zou2020gradient,lee2019wide}). In Section~\ref{sec:lowerbound}, we will see that such a ball is also large enough to cover the  whole optimization path of GD/SGD for the general feedforward neural networks. Hence, to study  the optimization dynamics of wide feedforward neural networks, this ball is large enough.

\end{remark}

To prove Eq.~(\ref{eq:lag_re}), we make an assumption on the width $m$:
\begin{assumption}\label{assump:md}
The maximum in-degree of any neuron is at most polynomial in the network width $m$:
\begin{align*}
    \sup_{\ell\in \{2,\ldots,L\},i\in [d_\ell]}m_i^{(\ell)} = O(m^c),
\end{align*} where $c>0$ is a constant.
\end{assumption}

 This assumption puts a constraint on the neurons with large in-degrees such that the in-degrees cannot be super-polynomially large compared to $m$. A natural question is whether this constraint is necessary, for example, do our main results still hold in cases some in-degrees are exponentially large in $m$? While we believe the answer is positive, we need this assumption to apply the proof techniques. 
 Specifically, we apply  the tail bound for the norm of matrix Gaussian series~\cite{tropp2015introduction}, where there is a dimension factor equivalent to the number of weight parameters. Thus an exponentially large dimension factor would result in useless bounds. It is still an open question whether the dimension factor in the bound can be removed or moderated (see 
the discussion after Theorem 4.1.1 in~\cite{tropp2015introduction}).

With these assumptions, we are ready to present our main result:
\begin{theorem}[Scaling of the Hessian norm]\label{thm:bound_hessian}
 Suppose Assumption \ref{assump:input}, \ref{assump:sigma} and \ref{assump:md} hold.  Given a fixed $R>0$, for all $\rvw \in \mathsf{B}(\rvw_0,R)$, with probability at least $1-\exp(-\Omega(\log^2{m}))$ over the random initialization $\rvw_0$, each output neuron $f_k$ of a feedforward neural network satisfies
\begin{align}\label{eq:small_hessian}
     \left\|H_{f_k}({\rvw}) \right\| = O\left({(\log m +R)^{L^2}}/{\sqrt{m}}\right) = \tilde{O}\left({R^{L^2}}/{\sqrt{m}}\right),~~~k\in [d_\ell].
\end{align}
\end{theorem}

This theorem states that the Hessian matrix, as the second derivative with respect to weight parameters $\rvw$, of any output neuron can be arbitrarily small, if the network width is sufficient large.

Note that Eq.~(\ref{eq:small_hessian}) holds for all $\rvw\in \mathsf{B}(\rvw_0,R)$ with high probability over the random initialization $\rvw_0$. The basic idea is that, the spectral norm of Hessian can be bounded at the center of the ball, i.e., $\rvw_0$, though probability bounds due to the randomness of $\rvw_0$. For all other points $\rvw\in\mathsf{B}(\rvw_0,R)$, 
 the distance $\|\rvw-\rvw_0\|$, being  no greater than $R$, controls $\|H(\rvw)-H(\rvw_0)\|$ such that it is no larger  than the order of $\|H(\rvw_0)\|$, hence $\|H(\rvw)\|$ keeps the same order.
See the proof details in Subsection~\ref{subsec:proof}.

Using the Taylor expansion Eq.~(\ref{eq:taylor_exp}), we can bound the Lagrange remainder and have transition to linearity of the network:
\begin{corollary}[Transition to linearity]\label{cor:taylor}
Suppose Assumption \ref{assump:input}, \ref{assump:sigma} and \ref{assump:md} hold.  Given a fixed $R>0$, for all $\rvw \in \mathsf{B}(\rvw_0,R)$, with probability at least $1-\exp(-\Omega(\log^2{m}))$ over the random initialization  $\rvw_0$,
each $f_k$ will be closely approximated by a linear model:
\begin{align*}
     |f_k(\rvw) - (f_k)_{\mathrm{lin}}(\rvw)| \leq \frac{1}{2}\sup_{\rvw \in \mathsf{B}(\rvw_0,R)}\|H_{f_k}(\rvw)\| R^2= \tilde{O}\left({R^{L^2+2}}/{\sqrt{m}}\right).
\end{align*}
\end{corollary}\vspace{-8pt}
For feedforward neural networks with multiple output neurons, the property of transition to linearity holds with high probability, if the number of output neurons is bounded, i.e., $d_\ell = O(1)$. See the result in Appendix~\ref{sec:multiple_output}.

Furthermore, as defined in Definition~\ref{def:fnn}, each pre-activation, as a function of all related weight parameters and inputs, can be viewed as a feedforward neural network. Therefore, we can  apply the same techniques used for Theorem~\ref{thm:bound_hessian} to show that each pre-activation can transition to linearity:

\begin{theorem}\label{cor:bound_each_hessian}
Suppose Assumption \ref{assump:input}, \ref{assump:sigma} and \ref{assump:md} hold.   Given a fixed radius $R>0$, for all $\rvw \in \mathsf{B}(\rvw_0,R)$, with probability at least $1-\exp(-\Omega(\log^2{m}))$ over the random initliazation of $\rvw_0$,  any pre-activation in a feedforward neural network i.e., $\tilde{f}_k^{(\ell)}(\rvw)$ satisfies
\begin{align}
     \left\|H_{\tilde{f}^{(\ell)}_k}({\rvw}) \right\| =O\left({(\log m +R)^{\ell^2}} / {\sqrt{m}}\right)= \tilde{O}\left({R^{\ell^2}}/{\sqrt{m}}\right),~~~\ell\in [L],~~k\in {[d_\ell]}.
\end{align}
\end{theorem}
\begin{remark}
Note that pre-activations in the input layer i.e., the input data and in the first layer are constant and linear functions respectively, hence the spectral norm of their Hessian is zero.
\end{remark}

\paragraph{Experimental verification.} To verify our theoretical result on the scaling of the Hessian norm, i.e., Theorem~\ref{thm:bound_hessian}, we train a  DAG network built from a 3-layer DenseNet with each weight removed i.i.d. with probability $1/2$,  on $10$ data points of CIFAR-2 (2-class subset of CIFAR-10~\cite{krizhevsky2009learning}) using GD.  We compute the maximum relative change of the tangent kernel (definition is deferred to~Eq.~(\ref{eq:ntk})) during training, i.e., $\max_t\|K_t-K_0\|/\|K_0\|$ to simulate the scaling of the spectral norm of the Hessian. We observe the convergence of loss for all widths $\{2,2^2,...,2^{12}\}$, and the scaling of the Hessian follows close to the theoretical prediction of $\Theta(1/\sqrt{m})$. See Figure~\ref{fig:exp}.

\begin{figure}
\centering
\begin{minipage}{.48\textwidth}
  \centering
  \vspace{-12pt}
  \includegraphics[width=0.9\linewidth]{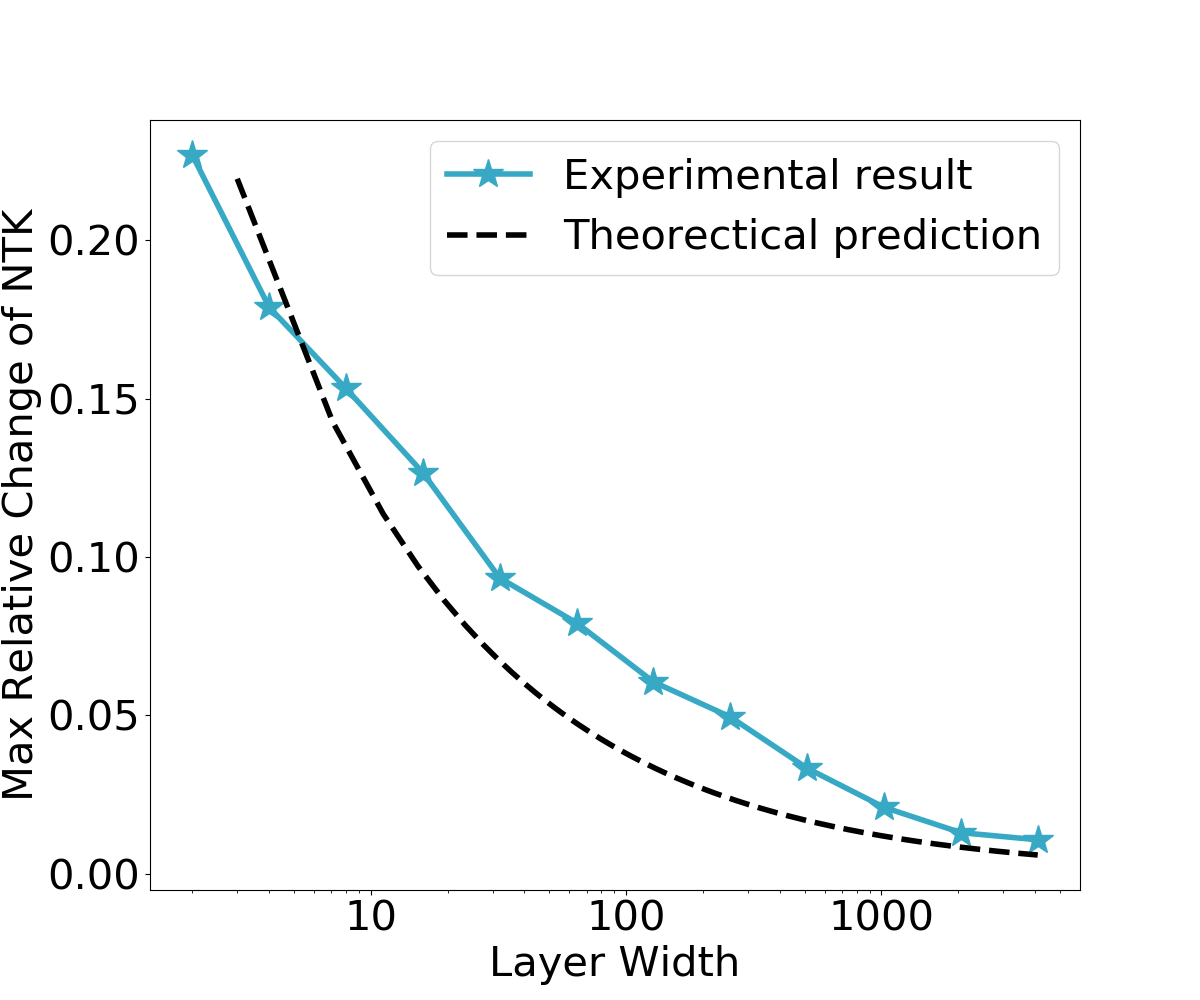} 
  \captionof{figure}{{\bf Transition to linearity of DAG network.} The experimental result approximates well the theoretical prediction of relative change of  tangent kernel from initialization to convergence, as a function of the network width. Each point on the solid curve is the average of independent 5 runs.}
  \label{fig:exp}
\end{minipage}\hfill
\begin{minipage}{.48\textwidth}
  \centering
  \includegraphics[width=0.85\linewidth]{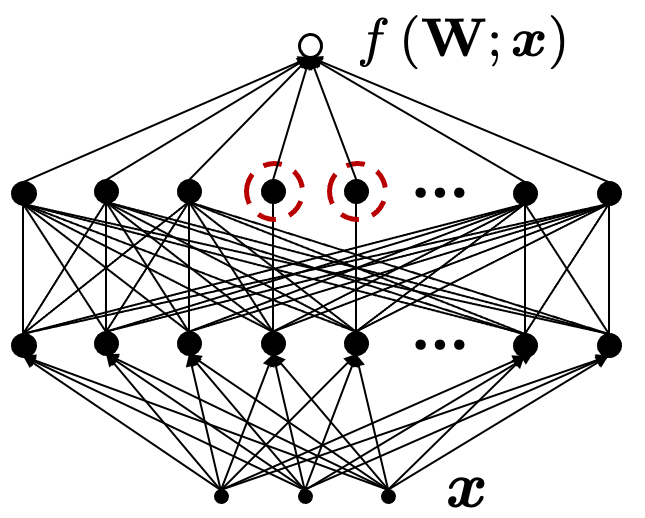}
  \vspace{10pt}
  \captionof{figure}{{\bf An example of DAG network with bottleneck neurons.} The DAG network $f(\rmW;\vx)$ has two bottleneck neurons (in red dashed circles) with in-degree 1, while the rest of neurons except for the input and the first layer have large in-degree. In this case,  $f(\rmW;\vx)$ will still transition to linearity with respect to $\rmW$ as the number of neurons goes to infinity.}
  \label{fig:bottleneck_neuron}
\end{minipage}
\vspace{-10pt}
\end{figure}

\paragraph{Non-linear activation on output neurons breaks transition to linearity.} In the above discussions, the transition to linearity of networks are under the assumption of identity activation function on every output neuron. In fact, the activation function on output neurons is critical to the linearity of neural networks. Simply, composing a non-linear function with a linear function will break the linearity.
Consistently, as shown in~\cite{liu2020loss} for FCNs, with non-linear activation function on the  output, transition to linearity does not hold any more.

\paragraph{``Bottleneck neurons'' do not necessarily break transition to linearity.}
We have seen that if {\it all}  neurons have sufficiently large in-degree, the network will transition to linearity. Does transition to linearity still hold, if neurons with small in-degree exist? 
We call  neurons with small in-degree {\it bottleneck neurons}, as their in-degrees are smaller than rest of neurons hence forming ``bottlenecks''.
As we show in Appendix E, existence of these bottleneck neurons does not break the transition to linearity, as long as the number of such neurons is significantly smaller than that of non-bottleneck neurons.
Figure~\ref{fig:bottleneck_neuron} shows an example with two bottleneck neurons, whose in-degree is $1$. This network still transitions to linearity as the number of neurons goes to infinity.

\subsection{Proof sketch of Theorem~\ref{thm:bound_hessian}}\label{subsec:proof}

By Lemma~\ref{lemma:sum_hessian}, the spectral norm of $H_{f_k}$ can be bounded by the summation of the spectral norm of all the Hessian blocks, i.e., $\|H_{f_k}\| \leq \sum_{\ell_1,\ell_2} \| H_{f_k}^{(\ell_1,\ell_2)}\|$, where $H_{f_k}^{(\ell_1,\ell_2)}:= \frac{\partial^2 f_k}{\partial \rvw^{(\ell_1)}\partial \rvw^{(\ell_2)}}$. Therefore, it suffices to bound the spectral norm of each block.
Without lose of generality, we consider the block  with $1\leq \ell_1\leq \ell_2\leq L$. 

By the chain rule of derivatives, we can write the Hessian block into:
\begin{align}\label{eq:hessian_block}
    \frac{\partial^2 f_k}{\partial \rvw^{(\ell_1)}\partial \rvw^{(\ell_2)}} = \sum_{\ell'=\ell_2}^{L}\sum_{i=1}^{d_{\ell'}}\frac{\partial^2 f_i^{(\ell')}}{\partial \rvw^{(\ell_1)}\partial \rvw^{(\ell_2)}}\frac{\partial f_k}{\partial f_i^{(\ell')}}
    := \sum_{\ell'=\ell_2}^L G^{L,\ell'}_{k}.
\end{align}

For each $G_k^{L,\ell'}$, since $f_i^{(\ell')} = \sigma\left(\tilde{f}_i^{(\ell')}\right)$, again by the chain rule of derivatives, we have
\begin{align}\label{eq:deri_G}
  G^{L,\ell'}_{k}&= \sum_{i=1}^{d_{\ell'}} \frac{\partial^2  \tilde{f}_i^{(\ell')}}{\partial  \rvw^{(\ell_1)}\partial \rvw^{(\ell_2)}}\frac{\partial f_k}{\partial  \tilde{f}_i^{(\ell')}} + \frac{1}{\sqrt{m_k^{(L)}}}\sum_{i: f_i^{(\ell')} \in \F_{\S^{(L)}_k}} \left(\rvw_k^{(L)}\right)_{\id^{L,k}_{\ell',i}} \sigma''\left(\tilde{ f}^{(\ell')}_i\right) \frac{\partial \tilde{ f}^{(\ell')}_i}{\partial \rvw^{(\ell_1)}}\left(\frac{\partial \tilde{ f}^{(\ell')}_i}{\partial \rvw^{(\ell_2)}}\right)^T\nonumber\\
    &= \frac{1}{\sqrt{m_k^{(L)}}}\sum_{r = \ell'}^{L-1} \sum_{i: f_i^{(r)}\in \F_{\S_k^{(L)}}}  \left(\rvw_k^{(L)}\right)_{\id^{L,k}_{r,i}} \sigma'\left(\tilde{f}_s^{(r)}\right) G^{r,\ell'}_i\nonumber\\
    &~~~~~+ \frac{1}{\sqrt{m_k^{(L)}}}\sum_{i: f_i^{(\ell')} \in \F_{\S^{(L)}_k}} \left(\rvw_k^{(L)}\right)_{\id^{L,k}_{\ell',i}} \sigma''\left(\tilde{ f}^{(\ell')}_i\right) \frac{\partial \tilde{ f}^{(\ell')}_i}{\partial \rvw^{(\ell_1)}}\left(\frac{\partial \tilde{ f}^{(\ell')}_i}{\partial \rvw^{(\ell_2)}}\right)^T,
\end{align}
where $\F_{\S_k^{(L)}} := \{ f: f\in f_{\S_k^{(L)}}\}$ and $\id^{L,k}_{\ell',i}:=\{p: \left(f_{\S_k^{(L)}}\right)_p = f_i^{(\ell')}\}$.

The first quantity on the RHS of the above equation, $\sum \left(\rvw_k^{(L)}\right)_{\id^{L,k}_{\ell',i}} \sigma'\left(\tilde{f}_i^{(r)}\right) G^{r,\ell'}_i$, is a matrix Gaussian series with respect to random variables $\rvw_k^{(L)}$, conditioned on fixed $\sigma'\left(\tilde{f}_i^{(r)}\right) G^{r,\ell'}_i$ for all $i$ such that $f_i^{(r)}\in \F_{\S_k^{(\ell)}}$.  We apply the tail bound for matrix Gaussian series, Theorem 4.1.1 from~\cite{tropp2015introduction}, to bound this quantity. To that end,  we need to  bound  its matrix variance, which   suffices to bound the spectral norm of $\sum_i G_i^{r,\ell'}$ since $\sigma'(\cdot)$ is assumed to be uniformly bounded by Assumption~\ref{assump:sigma}. There is a recursive relation that the norm bound of $G_k^{L,\ell'}$ depends on the norm bound of $G_i^{r,\ell'}$ which appears in the matrix variance. Therefore, we can recursively apply the argument to bound each $G$.

Similarly, the second quantity on the RHS of the above equation is also a matrix Gaussian series with respect to $\rvw_k^{(L)}$, conditioned on fixed $ \sigma''\left(\tilde{f}_i^{(\ell')}\right)\frac{\partial \tilde{ f}^{(\ell')}_i}{\partial \rvw^{(\ell_1)}}\left(\frac{\partial \tilde{ f}^{(\ell')}_i}{\partial \rvw^{(\ell_2)}}\right)^T$ for all $i$ such that $f_i^{(\ell')} \in \F_{\S^{(L)}_k}$.  As $\sigma''(\cdot)$ is assumed to be uniformly bounded by Assumption~\ref{assump:sigma}, we use Lemma~\ref{lemma:2_F_norm} to bound its matrix variance, hence the matrix Gaussian series can be bounded.

Note that such tail bound does not scale with the largest in-degree of the networks, since the in-degree of $f_k$, i.e., $m_k^{(L)}$,  will be cancelled out with the scaling factor $\nicefrac{1}{\sqrt{m_k^{(L)}}}$ in the bound of matrix variance.

See the complete proof in Appendix~\ref{sec:mainproof}.

\section{Relation to optimization }\label{sec:lowerbound}

While transition to linearity is a significant and surprising property of wide networks in its own right, it also  plays an important role in building the optimization theory of wide feedforward neural networks. 
Specifically, transition to linearity provides a path toward showing that the corresponding loss function satisfies the PL$^*$ condition in a ball of a certain radius, which is sufficient for exponential convergence of optimization to a global minimum by gradient descent or SGD~\cite{liu2020loss}.

Consider a supervised learning task. Given training inputs and labels  $\{(\vx_i,y_i)\}_{i=1}^n$, we use GD/SGD to minimize the square loss:
\begin{align}\label{eq:loss}
    \l(\rvw) = \frac{1}{2} \sum_{i=1}^n(f(\rvw;\vx_i)-y_i)^2,
\end{align}
where $f(\rvw;\cdot)$ is a  feedforward neural network.

The loss $\l(\rvw)$ is said to satisfy $\mu$-PL$^*$ condition, a variant of the well-known Polyak-{\L}ojasiewicz condition~\cite{polyak1963gradient,lojasiewicz1963topological}, at point $\rvw$, if
\begin{align*}
    \|\nabla_{\rvw} \l(\rvw)\|^2 \geq 2\mu \l(\rvw), ~~~\mu>0.
\end{align*}
Satisfaction of this $\mu$-PL$^*$ condition in a ball $\mathsf{B}(\rvw_0,R)$ with $R = O(1/\mu)$ around the starting point $\rvw_0$ of GD/SGD  guarantees a fast converge of the algorithm to a global minimum in this ball~\cite{liu2020loss}.

In the following, we use the transition to linearity of wide feedforward networks to 
establish the PL$^*$ condition for $\l(\rvw)$. Taking derivative on Eq.~(\ref{eq:loss}), we get
\begin{align}\label{eq:pl_*}
    \|\nabla_{\rvw} \l(\rvw)\|^2 \geq 2\lambda_{\min}(K(\rvw))\l(\rvw),
\end{align}
where matrix $K(\rvw)$, with elements 
\begin{align}\label{eq:ntk}
    K_{i,j}(\rvw) = \nabla_\rvw f(\rvw;\vx_i)^T\nabla_\rvw f(\rvw;\vx_j)~~\mathrm{for}~~ i,j\in[n],
\end{align}
 is called Neural Tangent Kernel (NTK)~\cite{jacot2018neural},
and $\lambda_{\min}(\cdot)$ denotes the smallest eigenvalue of a matrix. Note that, by definition, the NTK matrix is always positive semi-definite, i.e., $\lambda_{\min}(K(\rvw))\ge 0$.

Directly by the definition of PL$^*$ condition, at a given point $\rvw$, if $\lambda_{\min}(K(\rvw))$ is strictly positive, then the loss function $\l(\rvw)$ satisfies  PL$^*$ condition.

To establish convergence  of GD/SGD, it is sufficient to verify that  PL$^*$ condition is satisfied in a ball $\mathsf{B}(\rvw_0,R)$ with $R = O(1/\mu)$. Assuming that $\lambda_{\min}(K(\rvw_0))$ is bounded away from zero,  transition to linearity extends the satisfaction of the PL$^*$ condition from one point $\rvw_0$ to all  points in  $\mathsf{B}(\rvw_0,R)$.

\paragraph{PL$^*$ condition at $\rvw_0$.} For certain  neural networks, e.g., FCNs, CNNs and ResNets, strict positiveness of $\lambda_{\min}(K(\rvw_0))$ can be shown, see for example,~\cite{du2018gradient,du2018gradientdeep}. We expect same holds more generally,  in the case of general feedforward neural networks. 
Here we show that $\lambda_{\min}(K(\rvw_0))$ can be bounded from $0$ for one data point under certain additional assumptions.  Since there is only one data point, 
$\lambda_{\min}(K(\rvw_0)) = K(\rvw_0) = \|\nabla_\rvw f(\rvw_0)\|^2$.
We also assume the following on activation functions and the input.

 \smallskip
\begin{assumption}\label{assump:gaussian_input}
The input $\rvx$ satisfies $\rvx \sim \mathcal{N}(0,I_{d_0})$.
\end{assumption}
\begin{assumption}\label{assump:homo:act}
The activation function is homogeneous, i.e. $\sigma_i^{(\ell)}(az) = a^r \sigma_i^{(\ell)}(z), r>0$ for any constant $a$. And $\inf_{\ell\in[L-1],i\in [d_\ell]}\mathbb{E}_{z\sim\mathcal{N}(0,1)}\left[\sigma_i^{(\ell)}(z)^2\right] = C_\sigma>0$.
\end{assumption}
\smallskip
\begin{remark}
Here for simplicity we assume the activation functions are homogeneous with the same $r$. It is not hard to extend the result to the case that each activation function has different $r$.
\end{remark}
\begin{proposition}\label{prop:lowerbound_k}
With Assumption~\ref{assump:gaussian_input} and \ref{assump:homo:act}, we have for any $k\in {[d_\ell]}$,
\begin{align}
    \mathbb{E}_{\rvx,\rvw_0}[\|\nabla_\rvw f_k(\rvw_0)\|] &\geq \sqrt{\min\left(1,\min_{1\leq j\leq L}C_\sigma^{\sum_{l'=0}^{j-1}r^{\ell'}}\right)}= \Omega(1).
\end{align}
\end{proposition}
The proof can be found in Appendix~\ref{proof:lowerbound_k}.

The above proposition establishes a positive lower bound on $\|\nabla_\rvw f(\rvw_0)\|$, hence also on $\lambda_{\min}(K(\rvw_0))$. Using Eq.~(\ref{eq:pl_*}), we get that the loss function $\l(\rvw)$ satisfies PL$^*$ at $\rvw_0$.

 \paragraph{Extending PL$^*$ condition to $\mathsf{B}(\rvw_0,R)$.} Now we use transition to linearity to extend the satisfaction of PL$^*$ condition to the ball $\mathsf{B}(\rvw_0,R)$. In Theorem~\ref{thm:bound_hessian}, we see that, a feedforward neural network $f(\rvw)$ transitions to linearity, i.e.,  $\|H_f(\rvw)\| = \tilde{O}(1/\sqrt{m})$ in this ball.
 An immediate consequence is that, for any $\rvw\in\mathsf{B}(\rvw_0,R)$,
\begin{align*}
   | \lambda_{\min}(K(\rvw)) - \lambda_{\min}(K(\rvw_0))| \leq O\left(\sup_{\rvw\in \mathsf{B}(\rvw_0,R)}\|H_f(\rvw)\|\right).
\end{align*}
Since $\lambda_{\min}(K(\rvw_0))$ is bound from $0$ and $\|H_f(\rvw)\|$ can be arbitrarily small as long as $m$ is large enough, we have $\lambda_{\min}(K(\rvw))$ is lower bounded from $0$ in the whole ball. Specifically, there is a $\mu>0$ such that
\begin{align*}
    \inf_{\rvw\in \mathsf{B}(\rvw_0,R)}\lambda_{\min}(K(\rvw))\ge \mu.
\end{align*}
Moreover, the radius $R$ can be set to be $O(1/\mu)$, while keeping the above inequality hold.
Then, applying the theory in \cite{liu2020loss}, existence of global minima of $\l(\rvw)$ and convergence of GD/SGD can be established.

 For the case of multiple data points, extra techniques are needed to lower bound the minimum eigenvalue of the tangent kernel. Since we focus more on the transition to linearity of feedforward neural networks in this paper, we leave it as a future work.
 
\paragraph{Non-linear activation function on outputs and transition to linearity.} 
In this paper, we mainly discussed  feedforward neural networks with linear activation function on output neurons.
In most of the literature also considers this setting~\cite{jacot2018neural,montanari2020interpolation,nguyen2021tight,du2018gradient,du2018gradientdeep,zou2019improved,zou2020gradient}. In fact, as pointed out in~\cite{liu2020loss} for FCNs,  
while this linearity of activation function on the outputs is necessary for transition to linearity, it is not required for  successful  optimization. Specifically, simply adding a nonlinear activation function on the output layer causes the Hessian norm to be $O(1)$, independently of the network width. Thus transition to linearity does not occur. 
 However, the corresponding square loss can still satisfy the PL$^*$ condition and the existence of global minimum and efficient optimization can still be established.

\section{Discussion and future directions}
In this work, we showed that transition to linearity arises in  general feedforward neural networks with arbitrary DAG architectures, extending previous results for standard architectures~\cite{jacot2018neural, lee2019wide, liu2020linearity}. 
We identified the minimum in-degree of all neurons except for the input and first layers as the key quantity to control the transition to linearity of general feedforward neural networks. 

We showed that the property of transition to linearity is flexible to the choice of the neuron function Eq.~(\ref{eq:local_function}). For example, skip connections Eq.~(\ref{eq:res_op}) and shared weights Eq.~(\ref{eq:cnn_op}) do not break the property. Therefore, we believe our framework can be extended to more complicated neuron functions, e.g., attention layers~\cite{hron2020infinite}. 
For non-feedforward networks, such as RNN, recent work
\cite{alemohammad2020recurrent} showed that they  also have a constant NTK. For this reason,  we expect  transition to linearity also to occur for of non-feedforward networks.

Another direction of future work is better understanding of optimization for DAG networks, which requires a more delicate analysis of the NTK at initialization. Specifically, with multiple training examples, a lower bound on the minimum eigenvalue of the NTK of the DAG networks is sufficient for the PL$^*$ condition to hold, thus  guaranteeing the convergence of GD/SGD.

\section*{Acknowledgements}

We are grateful for support of
the NSF and the
Simons Foundation for the Collaboration on the Theoretical
Foundations of Deep Learning\footnote{\url{https://deepfoundations.ai/}}
through awards DMS-2031883 and \#814639. We also acknowledge NSF support through  IIS-1815697 and the TILOS institute (NSF CCF-2112665).
We thank Nvidia for the donation of GPUs.
This work used the Extreme Science and Engineering Discovery Environment (XSEDE,~\cite{xsede}),  which is supported by National Science Foundation grant number ACI-1548562 and allocation TG-CIS210104.

\printbibliography

\newpage
\appendix
\section*{Appendix}
\paragraph{Notations for set of neurons.}We extra define the following notations for the proof. For $0\leq \ell\leq L-1$, $i\in [d_\ell]$, we use $\F_{\S_i^{(\ell)}}$ to denote the set of all the elements in the vector $f_{\S_i^{(\ell)}}$ (Eq.~(\ref{eq:basic_op})):
\begin{align}\label{eq:set_elements}
    \F_{\S_i^{(\ell)}} := \{ f: f\in f_{\S_i^{(\ell)}}\}.
\end{align}

And we use $\P^{(\ell)}$ to denote the set of all neurons in $\ell'$-th layer i.e., $f^{(\ell')}$ defined in Eq.~(\ref{eq:layer}), with $0\leq \ell'\leq \ell$:
\begin{align}\label{eq:pool}
    \P^{(\ell)} := \{ f: f\in f^{(\ell')}, \ell'\leq \ell\}. 
\end{align}

\paragraph{Activation functions.} In Assumption~\ref{assump:sigma}, we assume the Lipschitz continuity and smoothness for all the activation functions. In the proof of lemmas, e.g., Lemma~\ref{lemma:2_F_norm} and ~\ref{lemma:2diff}, we only use the fact that they are Lipschitz continuous and smooth, as well as bounded by a constant $\gamma_0>0$ at point $0$, hence we use $\sigma(\cdot)$ to denote all the activation functions like what we do in Assumption~\ref{assump:sigma} for simplicity.

\paragraph{Notations for derivatives.}Additionally, in the following we introduce notations of the derivatives, mainly used in the proof of Lemma~\ref{lemma:2_F_norm} and Lemma~\ref{lemma:2diff}.

By definition of feedforward neural networks in Section~\ref{sec:general_nn}, different from the standard neural networks such as FCNs and CNNs in which the connection between neurons are generally only in adjacent layers, the neurons in feedforward neural networks can be arbitrarily  connected as long as there is no loop. 

To that end, we define ${\partial  f_{\S_i^{(\ell)}}}/{\partial  f^{(\ell')}}$ to be a mask matrix for any $\ell'< \ell$, $i\in [d_{\ell}]$ to indicate whether the neurons $ f_{\S_i^{(\ell)}}$ appear in  $ f^{(\ell')}$:
\begin{align}\label{eq:mask}
    \left( \frac{\partial  f_{\S_i^{(\ell)}}}{\partial  f^{(\ell')}}\right)_{j,k} =\mathbb{I}\left\{\left( f_{\S_i^{(\ell)}}\right)_k \in  f_j^{(\ell')}\right\}.
\end{align}

And $\partial  f_i^{(\ell)} / \partial  f_{\S_i^{(\ell)}}$ and $\partial  f_i^{(\ell)} / \partial \rvw_i^{(\ell)}$ are standard derivatives according to Eq.~(\ref{eq:basic_op}):
\begin{align*}
    \frac{\partial  f_i^{(\ell)}}{ \partial  f_{\S_i^{(\ell)}}} &= \frac{1}{\sqrt{m_i^{(\ell)}}}\left(\rvw_i^{(\ell)}\right)^T (\sigma_i^{(\ell)})'(\tilde{ f}_i^{(\ell)}),\\
    \frac{\partial  f_i^{(\ell)}}{ \partial \rvw_i^{(\ell)}} &= \frac{1}{\sqrt{m_i^{(\ell)}}}\left( f_{\S_i^{(\ell)}}\right)^T (\sigma_i^{(\ell)})'(\tilde{ f}_i^{(\ell)}).
\end{align*}

We give a table of notations that will be frequently used (See Table~\ref{table:notations}). The same notations will be used for ResNets and CNNs with extra subscripts $\res$ and $\cnn$ respectively.
\begin{table}[ht]
  \caption{Table of notations}
  \label{sample-table}
  \centering
  \begin{tabular}{ll}
    \toprule
    Symbol     & Meaning    \\
    \midrule
    $f^{(\ell)}$ &  Vector of neurons in $\ell$-th layer    \\
    $d_\ell$ &  Number of  neurons in $\ell$-th layer, i.e., length of $f^{(\ell)}$  \\
    $f_{\S_i^{(\ell)}}$ & Vector of in-coming neurons of $f_i^{(\ell)}$     \\
    $\rvw_i^{(\ell)}$ & Weight vector corresponding to in-coming edges of $f_i^{(\ell)}$ \\
    $m_i^{(\ell)}$ & Number of in-coming neurons of $f_i^{(\ell)}$, i.e., length of  $f_{\S_i^{(\ell)}}$ and $\rvw_i^{(\ell)}$\\
        $\sigma_i^{(\ell)}$ & Activation function on $\tilde{f}_i^{(\ell)}$\\
    $\rvw^{(\ell)}$ & Weight vector corresponding to all incoming edges toward neurons at layer $\ell$\\
    $\F_{\S_i^{(\ell)}}$ & Set of all the elements in the vector $f_{\S_i^{(\ell)}}$ (Eq.~(\ref{eq:set_elements}))\\
    $\P^{(\ell)}$ & Set of all neurons in $f^{(\ell')}$ with $0\leq \ell'\leq \ell$ (Eq.~(\ref{eq:pool}))\\
    $\id^{\ell_1,i}_{\ell_2,j}$ & Index of $f_j^{(\ell_2)}$ in the vector $f_{\S_i^{(\ell_1)}}$\\
    \bottomrule
  \end{tabular}\label{table:notations}
\end{table}
\section{Examples of feedforward neural networks}\label{app:example}

Here we show that many common neural networks are special examples of the feedforward neural networks in Definition~\ref{def:fnn}.

\paragraph{Fully-connected neural networks.} Given an input $\vx \in \mathbb{R}^{d}$, an $L$-layer fully-connected neural network is defined as follows:
\begin{align}\label{eq:fc}
 &f^{(0)} = \vx, \nonumber\\
 &f^{(\ell)}= \sigma\left(\frac{1}{\sqrt{m_{\ell-1}}}W^{(\ell)}f^{(\ell-1)}\right), \  \ \forall \ell \in[L-1],\\
 &f(\rmW;\vx) := f^{(L)} = \frac{1}{\sqrt{m_{L-1}}}W^{(L)}f^{(L-1)},\nonumber
\end{align} 
where each $f^{(\ell)}$ is a $m_\ell$-dimensional vector-valued function, and  $\rmW := \left(W^{(1)},...,W^{(\ell)}\right)$, $W^{(\ell)} \in \mathbb{R}^{m_{\ell+1}\times m_\ell}$, is the collection of all the weight matrices. Here $\sigma(\cdot)$ is an element-wise activation function, e.g., sigmoid function.

For FCNs, the inputs are the $0$-th layer neurons $f^{(0)} = \vx$ and the outputs are the $\ell$-th layer neurons $f^{(\ell)}$, which have zero in-degrees and zero out-degrees, respectively. For each non-input neuron, its in-degree is the number of neurons in its previous layer, $m_{\ell-1}$; the summation in Eq.~(\ref{eq:local_function}) turns out to be over all the neurons in the previous layer, which is manifested in the matrix multiplication of  $W^{(\ell)}f^{(\ell-1)}$. For this network, the activation functions are the same, except the ones on input and output neurons, where identity functions are used in the definition above.

\paragraph{ DenseNets~\cite{huang2017densely}.} Given an input $\vx \in \mathbb{R}^{d}$, an $L$-layer DenseNet is defined as follows:
\begin{align}
 &f^{(0)} = f_{\mathrm{temp}}^{(0)} = \vx, \nonumber\\
 &f^{(\ell)}= \sigma\left(\frac{1}{\sqrt{\sum_{l' = 0}^{\ell-1}m_{\ell'}}}W^{(\ell)}f_{\mathrm{temp}}^{(\ell-1)}\right), \label{eq:dense_neuron}  \\
 &f_{\mathrm{temp}}^{(\ell)} = \left[\left(f_{\mathrm{temp}}^{(\ell-1)}\right)^T,\left(f^{(\ell)}\right)^T\right]^T, \  \ \forall \ell\in[L-1],\nonumber \\
 &f(\rmW;\vx) :=f^{(L)}=\frac{1}{\sqrt{\sum_{\ell' = 0}^{L-1}m_{\ell'}}}W^{(L)}f_{\mathrm{temp}}^{(L-1)}\label{eq:densnet},
\end{align} 
where $\rmW = \left(W^{(1)},...,W^{(L)}\right)$ is the collection of all the weight matrices.  Here $\sigma(\cdot)$ is an element-wise activation function and for each $\ell\in[L]$, $W^{(\ell)} \in \mathbb{R}^{m_{\ell}\times \sum_{\ell'=0}^{\ell-1} m_{\ell'}}$.

The DenseNet shares much similarity with the fully-connected neural network, except that each non-input neuron depends on all the neurons in previous layers. This difference makes the in-degree of the neuron be $\sum_{\ell' = 0}^{\ell-1}m_{\ell'}$.

\paragraph{Neural networks with randomly dropped edges.} Given a network $f$ built from a DAG, for any neuron $f_v$ , where $v\in \mathcal{V}\backslash \mathcal{V}_{\mathrm{input}}$, according to Eq.~(\ref{eq:fc}), it is defined by 
\begin{align*}
        f_v =\sigma_v(\tilde{f}_v), \quad \tilde{f}_v= \frac{1}{\sqrt{\In(v)}}\sum_{u\in\mathcal{S}_{\mathrm{in}}(v)} w_{(u,v)}f_u.
\end{align*}
If each edge $(u,v)$ is randomly dropped with parameter $p\in(0,1)$, then the above equation becomes
\begin{align*}
          f_v =\sigma_v(\tilde{f}_v), \quad \tilde{f}_v= \frac{1}{\sqrt{\In(v)}}\sum_{u\in\mathcal{S}_{\mathrm{in}}(v)} w_{(u,v)}f_u\cdot \mathbb{I}_{\{\xi_{u,v} \geq p\}},  
\end{align*}
where $\xi_{u,v}$ is i.i.d. drawn from \textsf{Bernoulli}($p$).

To interpret such an architecture, we can simply remove the edges $(u,v)$ in the DAG where $\xi_{u,v}<p$. Then it is not hard to see that the new DAG network corresponds to the network with randomly dropped edges. 

Similarly, for a neural network with randomly dropped edges in multiple layers, we can remove all the edges whose corresponding $\xi$ is less than $p$. Then the resulting DAG can describe this network architecture.

We note the similarity of this network with the popularly used dropout layer~\cite{srivastava2014dropout}, both of which have a mechanism of randomly dropping out neurons/edges. However, the major difference is that, neural networks with dropout layers dynamically remove (or put mask on) neurons/edges during training, while the networks we considered only here drop edges and are fixed during training.

\section{Proof of  Theorem~\ref{thm:bound_hessian}}\label{sec:mainproof}
We will first compute the Hessian matrix of the network function  then show how to bound  the spectral norm of it.

We denote for each $\ell\in[L]$,
\begin{align}\label{eq:up_low_m}
    \underline{m}_{\,\ell} := \inf_{i\in [d_\ell]} m_i^{(\ell)},~~~\overline{m}_{\ell} := \sup_{i\in [d_\ell]} m_i^{(\ell)}.
\end{align}
By Assumption~\ref{assump:md}, it is not hard to infer that $\overline{m}_\ell$ and $\underline{m}_{\,\ell}$ are also polynomial in $m$.

Fixing $k\in [d_\ell]$, to bound $\|H_{f_k}\|$, we will first bound the spectral norm of the each Hessian block $H_{f_k}^{(\ell_1,\ell_2)}$, which takes the form 
 \begin{align*}
    H_{f_k}^{(\ell_1,\ell_2)} := \frac{\partial^2 f_k}{\partial \rvw^{(\ell_1)}\partial \rvw^{(\ell_2)}},~~k\in [d_\ell],~~ \ell_1,\ell_2\in [L].
\end{align*}
 
Without lose of generality, we assume $1\leq \ell_1\leq \ell_2\leq L$ and we start with the simple case when $\ell_2=L$.

If $\ell_1 = \ell_2= L$, $H^{(L,L)}_{f_k}$ is simply a zero matrix since $f_k(\rvw)$ is linear in $\rvw^{(\ell)}$.

If $1\leq \ell_1<\ell_2 = L$, we will use the following Lemma:
\begin{lemma}\label{lemma:2_F_norm}
Given $\ell'\geq 1$,  for any $\ell'+1 \leq \ell \leq L$,  $\rvw\in \mathsf{B}(\rvw_0,R)$, and $j \in [d_{\ell}]$, we have, with probability at least $1-\exp(-C_{\ell,\ell'}^ f \log^{2}m)$,
\begin{align}
   &\left\|\frac{\partial f_{S_j^{(\ell)}}}{\partial \rvw^{(\ell')}}\right\|= {O}\left(\max_{\ell'+1\leq p\leq \ell}\frac{\sqrt{m_j^{(\ell)}}}{\sqrt{\underline{m}_{\,p}}}(\log m +R)^{\ell'}\right) = \tilde{O}\left(\max_{\ell'+1\leq p\leq \ell}\frac{\sqrt{m_j^{(\ell)}}}{\sqrt{\underline{m}_{\,p}}}R^{\ell'}\right), \label{eq:2norm}\\
   &\left\|\frac{\partial f_{S_j^{(\ell)}}}{\partial \rvw^{(\ell')}}\right\|_F= {O}\left(\sqrt{m_j^{(\ell)}}(\log m+R)^{\ell-1}\right) = \tilde{O}\left(\sqrt{m_j^{(\ell)}}R^{\ell-1}\right)\label{eq:Fnorm},
\end{align}
where $C_{\ell,\ell'}^ f >0$ is a constant.
\end{lemma}
See the proof in Appendix~\ref{proof:2_F_norm}.

By Lemma~\ref{lemma:2_F_norm}, with probability at least $1-\exp(-\Omega(\log^2{m}))$,
\begin{align*}
    \left\|H_{f_k}^{(\ell_1,L)}\right\| &= \left\|\frac{1}{\sqrt{m_k^{(\ell)}}}\frac{\partial  f_{\S^{(\ell)}_k}}{\partial \rvw^{(\ell_1)}}\right\|= {O}\left(\max_{\ell_1+1 \leq \ell\leq L}\frac{1}{\sqrt{\underline{m}_{\,\ell}}}(\log m + R)^{\ell_1}\right) = \tilde{O}(R^{\ell_1}/\sqrt{m}).
\end{align*}

For the rest of blocks that $1\leq \ell_1\leq \ell_2\leq L-1$, we will use the following lemma to bound their spectral norms:

\begin{lemma}\label{lemma:2diff}
Given $1\leq \ell_1\leq \ell_2 \leq L-1$, for any $\ell_2< \ell\leq L$, $\rvw \in \mathsf{B}(\rvw_0,R)$,  and $j\in [d_{\ell}]$, we have, with probability at least $1-\exp(-\Omega(\log^2 m))$,
\begin{align}\label{eq:2diff_each}
    \left\|\frac{\partial^2 \tilde{f}_j^{(\ell)}}{\partial \rvw^{(\ell_1)}\partial \rvw^{(\ell_2)}}\right\|=  {O}\left(\max_{\ell_1+1\leq p\leq \ell}\frac{1}{\sqrt{\underline{m}_{\,p}}}(\log m + R)^{\ell^2}\right) = \tilde{O}\left(\max_{\ell_1+1\leq p\leq \ell}\frac{R^{\ell^2}}{\sqrt{\underline{m}_{\,p}}}\right).
\end{align}
\end{lemma}
See the proof in Appendix~\ref{proof:2diff}.
\begin{remark}\label{remark:node_linear}
Note that the above results hold for any $\ell \leq L$. When $\ell=L$, $\tilde{ f}_j^{(\ell)} = f_j$ which is what we need to show the transition to linearity of $f_j$. When $\ell<L$, as discussed before, we can regard $\tilde{ f}_j^{(\ell)}$ as a function of its parameters. We note that  $\tilde{ f}_j^{(\ell)}$ with $\ell<L$ will also transition to linearity by applying the same analysis for $f_j$, which is the result of Theorem~\ref{cor:bound_each_hessian}.
\end{remark}

By letting $\ell=L$ in Lemma~\ref{lemma:2diff}, for any $1\leq \ell_1\leq \ell_2\leq L-1$, with probability at least $1-\exp(-\Omega(\log^2{m}))$,
\begin{align*}
    \left\|H_{f_k}^{(\ell_1,\ell_2)}\right\|={O}\left(\max_{\ell_1\leq \ell \leq L}\frac{1}{\sqrt{\underline{m}_{\,\ell}}} (\log m+R)^{\ell^2}\right)
    = {O}((\log m +R)^{L^2}/\sqrt{m}) = \tilde{O}(R^{L^2}/\sqrt{m}).
\end{align*}

Finally by Lemma~\ref{lemma:sum_hessian}, the spectral norm of $H_{f_k}$ can be bounded by the summation of the spectral norm of all the Hessian blocks, i.e.,$\|H_{f_k}\| \leq \sum_{\ell_1,\ell_2} \| H_{f_k}^{(\ell_1,\ell_2)}\|$. Applying the union bound over the indices of layers $\ell_1,\ell_2$, we finish the proof.

\section{Feedforward neural networks with multiple output}\label{sec:multiple_output}
In cases of multiple output neurons, the network function is vector-valued and its Hessian is a three-order tensor. The spectral norm of Hessian is defined in a standard way, i.e., $$\|\rmH_f(\rvw)\| := \sup_{\|\rvv\|=\|\rvu\|=\|\rvs\|=1}\sum_{i,j,k}\left(\rmH_f(\rvw)\right)_{i,j,k}v_i u_j s_k,$$ where $\rvs\in \mathbb{R}^{d_\ell}$ and $\rvv$, $\rvu$ have the same dimension with $\rvw$. It is not hard to see that $\|\rmH_f(\rvw)\| \leq d_\ell \max_{k\in[d_\ell]} \|H_{f_{k}}(\rvw)\|$.

If the number of output neurons $d_\ell$ is bounded (as in most practical cases), the spectral norm of the Hessian of $f$ is also of the order $\tilde{O}(1/\sqrt{m})$, with high probability, as a direct consequence of Theorem~\ref{thm:bound_hessian}.

\begin{corollary}\label{cor:bound_multi_hessian}
Suppose Assumption \ref{assump:input}, \ref{assump:sigma} and \ref{assump:md} hold.   Given a fixed radius $R>0$, for all $\rvw \in \mathsf{B}(\rvw_0,R)$, with probability at least $1-\exp(-\Omega(\log^2{m}))$ over the random initialization $\rvw_0$,  a vector-valued feedforward neural network $f$ satisfies
\begin{align}
     \left\|\rmH_{f}({\rvw}) \right\| = \tilde{O}\left(\frac{R^{L^2}}{\sqrt{m}}\right).
\end{align}
\end{corollary}

\section{Feedforward neural networks with skip connections}\label{sec:skip_connection}
 
In this section, we discuss the property of transition to linearity holds for networks with skip connection.

We formally define the skip connection in the following. 
We add a skip connection to each neuron then 
the neuron functions Eq.~(\ref{eq:basic_op}) become 
\begin{align}\label{eq:res_op}
    f_{i,\res}^{(\ell)} =\sigma_{i}^{(\ell)}\left(\tilde{f}_{i,\res}^{(\ell)}\right) + {f}_{{B}(\ell,i),\res}^{({A}(\ell,i))}, \ \  \tilde{f}_{i,\res}^{(\ell)}= \frac{1}{\sqrt{m_{i}^{(\ell)}}}\left(\rvw_i^{(\ell)}\right)^T f_{\S^{(\ell)}_i,\res} ,
\end{align}
where $1\leq \ell \leq L-1$ and $i\in[d_{\ell}]$. Here ${A}(\ell,i) \in \{0,\cdots,\ell-1\}$ denotes the layer index of the connected neuron by skip connection with respect to $f_{i,\res}^{(\ell)}$ and ${B}(\ell,i)\in [d_{{A}(\ell,i)}]$.

And for the output layer $\ell = L$, we define
\begin{align*}
    {f}_{i,\res}^{(L)} = \tilde{f}_{i,\res}^{(L)}= \frac{1}{\sqrt{m_{i}^{(L)}}}\left(\rvw_i^{(L)}\right)^T f_{\S^{(L)}_i,\res} ,
\end{align*}
where $i\in[d_L]$.

The following theorem shows the property of transition to linearity holds for networks with skip connections. The proof of the  theorem follows the almost identical idea with the proof of 
Theorem~\ref{thm:bound_hessian}, hence we present the proof sketch and focus on the arguments that are new for $f_{\res}$.

\begin{theorem}[Scaling of the Hessian norm for $f_{\res}$]\label{thm:skip_connection}
Suppose Assumption \ref{assump:input}, \ref{assump:sigma} and \ref{assump:md} hold.   Given a fixed radius $R>0$, for all $\rvw \in \mathsf{B}(\rvw_0,R)$, with probability at least $1-\exp(-\Omega(\log^2{m}))$ over the random initliazation of $\rvw_0$, each output neuron $f_{k,\res}$  satisfies 
\begin{align}
     \left\|H_{{f}_{k,\res}}({\rvw}) \right\| = \tilde{O}\left(\frac{R^{L^2}}{\sqrt{m}}\right),~~~\ell\in [L],~~k\in {[d_\ell]}.
\end{align}
\end{theorem}

\begin{proof}[Proof sketch of Theorem~\ref{thm:skip_connection}]

For each output $f_{k,\res}$, where $k\in[d_\ell]$, similar to the proof of Theorem~\ref{thm:bound_hessian}, we bound the spectral norm of each Hessian block, i.e., $\frac{\partial^2 f_{k,\res}}{\partial \rvw^{(\ell_1)}\partial \rvw^{(\ell_2)}}$. Without loss of generality, we assume $1\leq \ell_1\leq \ell_2\leq L.$

Similar to Eq.(\ref{eq:hessian_block}), we derive the expression of the Hessian block by definition:
\begin{align*}
    \frac{\partial^2 f_{k,\res}}{\partial \rvw^{(\ell_1)}\partial \rvw^{(\ell_2)}} &= \sum_{\ell'=\ell_2}^{L}\sum_{i=1}^{d_{\ell'}}\frac{\partial^2 f_{i,\res}^{(\ell')}}{\partial \rvw^{(\ell_1)}\partial \rvw^{(\ell_2)}}\frac{\partial f_{k,\res}}{\partial f_{i,\res}^{(\ell')}} := \sum_{\ell'=\ell_2}^L G^{L,\ell'}_{k,\res}.
\end{align*}

And again by chain rule of derivatives, each $G^{L,\ell'}_{k,\res}$ can be written as
\begin{align*}
   G^{L,\ell'}_{k,\res}&=  \underbrace{ \frac{1}{\sqrt{m_k^{(L)}}}\sum_{r = \ell'}^{L-1} \sum_{s: f_{s,\res}^{(r)}\in \F_{\S_k^{(L)}}}  \left(\rvw_k^{(L)}\right)_{\id^{L,k}_{r,s}} \sigma'\left(\tilde{f}_{s,\res}^{(r)}\right) G^{r,\ell'}_{s,\res}}_{T_1}\\
    &~~~~+\underbrace{\frac{1}{\sqrt{m_k^{(L)}}}\sum_{r = \ell'}^{L-1} \sum_{s: f_{s,\res}^{(r)}\in \F_{\S_k^{(L)}}}  \left(\rvw_k^{(L)}\right)_{\id^{L,k}_{r,s}} \sigma'\left(\tilde{f}_{B(\ell',s),\res}^{(A(\ell',s))}\right) G^{A(\ell',s),\ell'}_{B(\ell',s),\res}}_{T_2} \\
   &~~~~+ \underbrace{\frac{1}{\sqrt{m_k^{(L)}}}\sum_{i: f_{i,\res}^{(\ell')} \in \F_{\S^{(L)}_{k,\res}}} \left(\rvw_k^{(L)}\right)_{\id^{L,k}_{\ell',i}} \left(\sigma''\left(\tilde{ f}^{(\ell')}_{i,\res}\right) \frac{\partial \tilde{ f}^{(\ell')}_{i,\res}}{\partial \rvw^{(\ell_1)}}\left(\frac{\partial \tilde{ f}^{(\ell')}_{i,\res}}{\partial \rvw^{(\ell_2)}}\right)^T \right)}_{T_3}\\
   &~~~~+\underbrace{\frac{1}{\sqrt{m_k^{(L)}}}\sum_{i: f_{i,\res}^{(\ell')} \in \F_{\S^{(L)}_{k,\res}}} \left(\rvw_k^{(L)}\right)_{\id^{L,k}_{\ell',i}} \left(\sigma''\left(\tilde{ f}^{(A(\ell',i))}_{B(\ell',i),\res}\right) \frac{\partial \tilde{ f}^{(A(\ell',i))}_{B(\ell',i),\res}}{\partial \rvw^{(\ell_1)}}\left(\frac{\partial \tilde{ f}^{(A(\ell',i))}_{B(\ell',i),\res}}{\partial \rvw^{(\ell_2)}}\right)^T \right)}_{T_4},
\end{align*}
where  $\F_{\S_{k,\res}^{(L)}} := \{ f: f\in f_{\S_{k,\res}^{(L)}}\}$ and $\id^{L,k}_{\ell',i}:=\{p: \left(f_{\S_k^{(L)},\res}\right)_p = f_{i,\res}^{(\ell')}\}$.

Note that the new terms  which are induced by the skip connection  in the above equation are
\begin{align*}
 T_2 = \frac{1}{\sqrt{m_k^{(L)}}}\sum_{r = \ell'}^{L-1} \sum_{s: f_{s,\res}^{(r)}\in \F_{\S_k^{(L)}}}  \left(\rvw_k^{(L)}\right)_{\id^{L,k}_{r,s}} \sigma'\left(\tilde{f}_{B(\ell',s),\res}^{(A(\ell',s))}\right) G^{A(\ell',s),\ell'}_{B(\ell',s),\res},
\end{align*}
and
\begin{align*}
T_4 = \frac{1}{\sqrt{m_k^{(L)}}}\sum_{i: f_{i,\res}^{(\ell')} \in \F_{\S^{(L)}_{k,\res}}} \left(\rvw_k^{(L)}\right)_{\id^{L,k}_{\ell',i}} \left(\sigma''\left(\tilde{ f}^{(A(\ell',i))}_{B(\ell',i),\res}\right) \frac{\partial \tilde{ f}^{(A(\ell',i))}_{B(\ell',i),\res}}{\partial \rvw^{(\ell_1)}}\left(\frac{\partial \tilde{ f}^{(A(\ell',i))}_{B(\ell',i),\res}}{\partial \rvw^{(\ell_2)}}\right)^T \right).
\end{align*}

These two new terms take the same form with the original two terms i.e., $T_1$ and $T_3$, which are matrix Gaussian series with respect to the random variables $\rvw_k^{(L)}$. Therefore, we can use the same method as $T_1$ and $T_3$ to bound the spectral norm of $T_2$ and $T_4$. 

As $A(\ell',i) < \ell'$ by definition, the bound on $T_2$ and $T_4$ will be automatically included in our recursive analysis. Then the rest of the proof is identical to the one for feedforward neural networks, i.e., the proof of Theorem~\ref{thm:bound_hessian}.

\end{proof}

\section{Feedforward neural networks with shared weights, e.g., convolutional neural networks}\label{sec:cnn}

In this section, we consider the feedforward neural networks where weight parameters are shared, e.g., convolutional neural networks, as an extension to our result where we assume each weight parameter $w_e\in\mathcal{W}$ is initialized i.i.d.  We will show that such feedforward neural networks in which the weight parameters are shared constant times, i.e., independent of the width $m$, the property of transition to linearity still holds.

We formally define the networks with shared weights in the following:
\begin{align}\label{eq:cnn_op}
    f_{i,j,\cnn}^{(\ell)} = \sigma_{i}^{(\ell)}\left(\tilde{f}_{i,j,\cnn}^{(\ell)}\right),~~~\tilde{f}_{i,j,\cnn}^{(\ell)} = \frac{1}{\sqrt{m_{i,j}^{(\ell)}}}\left(\rvw_{i}^{(\ell)}\right)^T f_{\S_{i,j}^{(\ell)},\cnn},
\end{align}
where $1\leq \ell \leq L$, $i\in[d_\ell]$. We introduce new index  $j\in[D(\ell,i)]$ where  $D(\ell,i)$  denotes the number of times that weights $\rvw_i^{(\ell)}$ are shared. Note that the element in $f_{\S_{i,j}^{(\ell)},\cnn}$ is allowed to be $0$, corresponding to the zero padding which is commonly used in CNNs.

We similarly denote the output of the networks $f_{i,j,\cnn}^{(L)}$ by $f_{i,j,\cnn}$.

To see how CNNs fit into this definition, we consider a CNN with 1-D convolution as a simple example.

\paragraph{Convolutional neural networks} Given input $\vx \in \mathbb{R}^d$, an $\ell$-layer convolutional neural network is defined as follows:
\begin{align}\label{eq:cnn}
 & f^{(0)} = \vx, \nonumber\\
 & f^{(\ell)}= \sigma\left(\frac{1}{\sqrt{m_{\ell-1}\times p}}W^{(\ell)}\ast f^{(\ell-1)}\right), \  \ \forall l \in[L-1],\nonumber\\
 &f_i(\rmW;\vx) = \frac{1}{\sqrt{m_{\ell-1}\times d}}\left\langle W_{[i,:,:]}^{(\ell)}, f^{(\ell-1)}\right\rangle,\  \ \forall i \in[d_{\ell}],
\end{align} 
where $\rmW = \left(W^{(1)},...,W^{(\ell)}\right)$ is the collection of all the weight matrices. 

We denote the size of the window by $p\times 1$, hence $W^{(\ell)}\in\mathbb{R}^{m_\ell\times m_{\ell-1}\times p}$ for $\ell\in[L-1]$. We assume the stride is 1 for simplicity, and we do the standard zero-padding to each $ f^{(\ell)}$ such that for each $\ell\in [L-1]$, $ f^{(\ell)} \in \mathbb{R}^{m_\ell\times d}$.  At the last layer, as $ f^{(\ell-1)}\in \mathbb{R}^{m_{\ell-1} \times d}$ and $W^{(\ell)}\in\mathbb{R}^{m_{L}\times m_{\ell-1} \times d}$, we do the matrix inner product for each $i\in[d_{\ell}]$.

Now we show how Eq.~(\ref{eq:cnn}) fits into Eq.~(\ref{eq:cnn_op}). For $\ell \in[L-1]$, in Eq.~(\ref{eq:cnn}), each component of $f^{(\ell)}\in \mathbb{R}^{m_\ell\times d}$ is computed as
\begin{align*}
   f^{(\ell)}_{i,j} = \sigma\left(\frac{1}{\sqrt{m_{\ell-1}\times p}} \left\langle W^{(\ell)}_{[i,:,:]}, f_{\left[:,j-\ceil{\frac{p-1}{2}}:j+\ceil{\frac{p-1}{2}}\right]}^{(\ell-1)}\right\rangle\right).
\end{align*}
Therefore, $m_{i,j}^{(\ell)}$, $\rvw_i^{(\ell)}$ and $f_{\S_{i,j}^{(\ell)},\cnn}$ in Eq.~(\ref{eq:cnn_op}) correspond to $m_{\ell-1}\times p$, $W^{(\ell)}_{[i,:,:]}$ and $f_{\left[:,j-\ceil{\frac{p-1}{2}}:j+\ceil{\frac{p-1}{2}}\right]}^{(\ell-1)}$ respectively. For $\ell = L$, $m_{i,j}^{(L)}$ corresponds to $m_{L-1}\times d$ and $f_{\S_{i,j}^{(L)},\cnn}$
corresponds to $f^{(L-1)}$. Then we can see our definition of networks with shared weights, i.e., Eq.~(\ref{eq:cnn_op}) includes standard CNN as an special example. 

Similar to Theorem~\ref{thm:bound_hessian}, we will show that the spectral norm of its Hessian can be controlled, hence the property of transition to linearity will hold for $f^{(\ell)}_{\cnn}$. The proof of the following theorem follows the almost identical idea with the proof of 
Theorem~\ref{thm:bound_hessian}, hence we present the proof sketch and focus on the arguments that are new for $f_{\cnn}$.

\begin{theorem}\label{thm:hessian_cnn}
Suppose Assumption \ref{assump:input}, \ref{assump:sigma} and \ref{assump:md} hold.   Given a fixed radius $R>0$, for all $\rvw \in \mathsf{B}(\rvw_0,R)$, with probability at least $1-\exp(-\Omega(\log^2{m}))$ over the random initliazation of $\rvw_0$,  each output neuron $f_{i,j,\cnn}(\rvw)$ satisfies
\begin{align}\label{eq:hessian_bound_cnn}
     \left\|H_{{f}_{i,j,\cnn}}({\rvw}) \right\| =O\left({(\log m +R)^{\ell^2}} / {\sqrt{m}}\right)= \tilde{O}\left({R^{\ell^2}}/{\sqrt{m}}\right),~~~\ell\in [L],~~i\in {[d_\ell]},~~j\in[D(i,\ell)].
\end{align}
\end{theorem}

\begin{proof}[Proof sketch of Theorem~\ref{thm:hessian_cnn}]

Similar to the proof of Theorem~\ref{thm:bound_hessian}, by Lemma~\ref{lemma:sum_hessian}, the spectral norm of $H_{f_{i,j,\cnn}}$ can be bounded by the summation of the spectral norm of all the Hessian blocks, i.e., $\|H_{f_{i,j,\cnn}}\| \leq \sum_{\ell_1,\ell_2} \| H_{f_{i,j,\cnn}}^{(\ell_1,\ell_2)}\|$, where $H_{f_{i,j,\cnn}}^{(\ell_1,\ell_2)}:= \frac{\partial^2 f_k}{\partial \rvw^{(\ell_1)}\partial \rvw^{(\ell_2)}}$. Therefore, it suffices to bound the spectral norm of each block.
Without lose of generality, we consider the block  with $1\leq \ell_1\leq \ell_2\leq L$. 

By the chain rule of derivatives, we can write the Hessian block into:
\begin{align}
    \frac{\partial^2 f_{i,j,\cnn}}{\partial \rvw^{(\ell_1)}\partial \rvw^{(\ell_2)}} &= \sum_{\ell'=\ell_2}^{L}\sum_{k=1}^{d_{\ell'}}\sum_{t=1}^{D(k,\ell')}\frac{\partial^2 f_{k,t,\cnn}^{(\ell')}}{\partial \rvw^{(\ell_1)}\partial \rvw^{(\ell_2)}}\frac{\partial f_{i,j,\cnn}}{\partial f_{k,t,\cnn}^{(\ell')}} := \sum_{\ell'=\ell_2}^L G^{L,\ell'}_{i,j,\cnn}.
\end{align}

For each $ G^{L,\ell'}_{i,j,\cnn}$, since $f_{i,j,\cnn}^{(\ell')} = \sigma\left(\tilde{f}_{i,j,\cnn}^{(\ell')}\right)$, again by the chain rule of derivatives, we have
\begin{align*}
  G^{L,\ell'}_{i,j,\cnn}&= \sum_{k=1}^{d_{\ell'}}\sum_{t=1}^{D(k,\ell')} \frac{\partial^2  \tilde{f}_{k,t,\cnn}^{(\ell')}}{\partial  \rvw^{(\ell_1)}\partial \rvw^{(\ell_2)}}\frac{\partial f_{i,j,\cnn}}{\partial  \tilde{f}_{k,t,\cnn}^{(\ell')}}\\
  &~~~~+ \frac{1}{\sqrt{m_{i,j}^{(L)}}}\sum_{k,t: f_{k,t,\cnn}^{(\ell')} \in \F_{\S^{(L)}_{i,j,\cnn}}} \left(\rvw_i^{(L)}\right)_{\id^{L,i,j}_{\ell',k,t}} \sigma''\left(\tilde{ f}^{(\ell')}_{k,t,\cnn}\right) \frac{\partial \tilde{ f}^{(\ell')}_{k,t,\cnn}}{\partial \rvw^{(\ell_1)}}\left(\frac{\partial \tilde{ f}^{(\ell')}_{k,t,\cnn}}{\partial \rvw^{(\ell_2)}}\right)^T\\
    &= \frac{1}{\sqrt{m_{i,j}^{(L)}}}\sum_{r = \ell'}^{L-1} \sum_{k,t: f_{k,t,\cnn}^{(r)}\in \F_{\S_{i,j,\cnn}^{(L)}}}  \left(\rvw_i^{(L)}\right)_{\id^{L,i,j}_{r,k,t}} \sigma'\left(\tilde{f}_{k,t,\cnn}^{(r)}\right) G^{r,\ell'}_{k,t,\cnn}\\
    &~~~~~+ \frac{1}{\sqrt{m_{i,j}^{(L)}}}\sum_{k,t: f_{k,t,\cnn}^{(\ell')} \in \F_{\S^{(L)}_{i,j,\cnn}}} \left(\rvw_i^{(L)}\right)_{\id^{L,i,j}_{\ell',k,t}} \sigma''\left(\tilde{ f}^{(\ell')}_{k,t,\cnn}\right) \frac{\partial \tilde{ f}^{(\ell')}_{k,t,\cnn}}{\partial \rvw^{(\ell_1)}}\left(\frac{\partial \tilde{ f}^{(\ell')}_{k,t,\cnn}}{\partial \rvw^{(\ell_2)}}\right)^T,
\end{align*}
where $\F_{\S_{i,j,\cnn}^{(L)}} := \{ f: f\in f_{\S_{i,j,\cnn}^{(L)}}\}$ and $\id^{L,i,j}_{\ell',k,t}:=\{p: \left(f_{\S_{i,j}^{(L)},\cnn}\right)_p = f_{k,t,\cnn}^{(\ell')}\}$.

Compared to the derivation for standard feedforward neural networks, i.e., Eq.~(\ref{eq:deri_G}), there is an extra summation over the index $t$, whose carnality is at most $D(k,\ell')$. Recall that $D(k,\ell')$ denotes the number of times that the weight parameters $\rvw_k^{(\ell')}$ is shared. Therefore, as we assume $D(k,\ell')$ is independent of the width $m$, the norm bound will have the same order of $m$. Consequently, the spectral norm of each $G_{i,j,\cnn}^{L,\ell'}$ can be recursively bounded then Eq.~(\ref{eq:hessian_bound_cnn}) holds.

\end{proof}

\section{Feedforward neural networks with bottleneck neurons}\label{sec:bottleneck}
In this section, we show that constant number of bottleneck neurons which serve as incoming neurons will
not break the linearity.

We justify this claim based on the recursive relation in Eq.~(\ref{eq:hessian_block}), which is used to prove the small spectral norm of the Hessian of the network function, hence proving the transition to linearity.

Recall that each Hessian block can be written into:
\begin{align}
    \frac{\partial^2 f_k}{\partial \rvw^{(\ell_1)}\partial \rvw^{(\ell_2)}} &= \sum_{\ell'=\ell_2}^{L}\sum_{i=1}^{d_{\ell'}}\frac{\partial^2 f_i^{(\ell')}}{\partial \rvw^{(\ell_1)}\partial \rvw^{(\ell_2)}}\frac{\partial f_k}{\partial f_i^{(\ell')}} := \sum_{\ell'=\ell_2}^L G^{L,\ell'}_{k}.
\end{align}

For each $G_k^{L,\ell'}$,  we have a recursive form
\begin{align}
  G^{L,\ell'}_{k}&= \sum_{i=1}^{d_{\ell'}} \frac{\partial^2  \tilde{f}_i^{(\ell')}}{\partial  \rvw^{(\ell_1)}\partial \rvw^{(\ell_2)}}\frac{\partial f_k}{\partial  \tilde{f}_i^{(\ell')}} + \frac{1}{\sqrt{m_k^{(L)}}}\sum_{i: f_i^{(\ell')} \in \F_{\S^{(L)}_k}} \left(\rvw_k^{(L)}\right)_{\id^{L,k}_{\ell',i}} \sigma''\left(\tilde{ f}^{(\ell')}_i\right) \frac{\partial \tilde{ f}^{(\ell')}_i}{\partial \rvw^{(\ell_1)}}\left(\frac{\partial \tilde{ f}^{(\ell')}_i}{\partial \rvw^{(\ell_2)}}\right)^T\nonumber\\
    &= \frac{1}{\sqrt{m_k^{(L)}}}\sum_{r = \ell'}^{L-1} \sum_{i: f_i^{(r)}\in \F_{\S_k^{(L)}}}  \left(\rvw_k^{(L)}\right)_{\id^{L,k}_{r,i}} \sigma'\left(\tilde{f}_s^{(r)}\right) G^{r,\ell'}_i\nonumber\\
    &~~~~~+ \frac{1}{\sqrt{m_k^{(L)}}}\sum_{i: f_i^{(\ell')} \in \F_{\S^{(L)}_k}} \left(\rvw_k^{(L)}\right)_{\id^{L,k}_{\ell',i}} \sigma''\left(\tilde{ f}^{(\ell')}_i\right) \frac{\partial \tilde{ f}^{(\ell')}_i}{\partial \rvw^{(\ell_1)}}\left(\frac{\partial \tilde{ f}^{(\ell')}_i}{\partial \rvw^{(\ell_2)}}\right)^T,
\end{align}
where $\F_{\S_k^{(L)}} := \{ f: f\in f_{\S_k^{(L)}}\}$ and $\id^{L,k}_{\ell',i}:=\{p: \left(f_{\S_k^{(L)}}\right)_p = f_i^{(\ell')}\}$.

As mentioned in  Section~\ref{subsec:proof},  to prove the spectral norm of $G_k^{L,\ell'}$ is small, we need to bound the matrix variance, which suffices to bound the spectral norm of
\begin{align*}
    \frac{1}{\sqrt{m_k^{(L)}}}\sum_{i: f_i^{(r)}\in \F_{\S_k^{(L)}}} G_i^{r,\ell'}~~\mathrm{and}~~~\frac{1}{\sqrt{m_k^{(L)}}}\sum_{i: f_i^{(\ell')} \in \F_{\S^{(L)}_k}}\frac{\partial \tilde{ f}^{(\ell')}_i}{\partial \rvw^{(\ell_1)}}\left(\frac{\partial \tilde{ f}^{(\ell')}_i}{\partial \rvw^{(\ell_2)}}\right)^T.
\end{align*}
 
 For the first quantity, if  all $\tilde{f}_i^{(r)}$ are neurons with large in-degree, which is the case of our analysis by Assumption~\ref{assump:md}, then each $\tilde{f}_i^{(r)}$ will transition to linearity by Theorem~\ref{cor:bound_each_hessian}. This is manifested as small spectral norm of $G_i^{r,\ell'}$ for all $i$.  If some of $\tilde{f}_i^{(r)}$ are neurons with small in-degree, their corresponding $G_i^{r,\ell'}$ can be of a larger order, i.e., $O(1)$.  However,  note that the cardinally of the set $\F_{\S_k^{(L)}}$ is $m_k^{(L)}$. As long as the number of such neurons is not too large, i.e., $o\left(m_k^{(L)}\right)$, the order of the summation will be not affected.  Therefore, the desired bound for the matrix variance will be the same hence the recursive argument can still apply.
 
 The same analysis works for the second quantity as well. Neurons with small in-degree can make the norm of $\frac{\partial \tilde{ f}^{(\ell')}_i}{\partial \rvw^{(\ell_1)}}\left(\frac{\partial \tilde{ f}^{(\ell')}_i}{\partial \rvw^{(\ell_2)}}\right)^T$ be of a larger order. However, as long as the number of such neurons is not too large, the bound still holds.

 For example,  for the bottleneck neural network which  has a narrow hidden layer (i.e., bottleneck layer) while the rest of hidden layers are  wide,
all neurons in the next layer to the bottleneck layer are bottleneck neurons. Such bottleneck neural networks were shown to break transition to linearity in~\cite{liu2020linearity}. 
However, we observe that for such bottleneck neural networks, the number of bottleneck neurons is large, a fixed fraction of all neurons.  With our analysis, if we add trainable connections to the bottleneck neurons such that almost all (except a small number of) bottleneck neurons become neurons with sufficiently large in-degrees, then the resulting network can have the property of transition to linearity.

\section{Proof of Proposition~\ref{prop:lowerbound_k}}\label{proof:lowerbound_k}

Note that for any $k\in{[d_\ell]}$, 
\begin{align*}
   \|\nabla_\rvw f_k(\rvw_0)\| &\geq \|\nabla_{\rvw^{(\ell)}}f_k(\rvw_0)\|= \left\| \frac{1}{\sqrt{m_k^{(\ell)}}}  f_{\S_k^{(\ell)}}\right\|= \frac{1}{\sqrt{m_k^{(\ell)}}}\left\| f_{\S_k^{(\ell)}}\right\|.
\end{align*}
Since $ f_{\S_{k}^{(\ell)}}$ contains neurons from $\P^{(\ell)}$ (defined in Eq.~(\ref{eq:pool})), in the following we prove $\mathbb{E}_{\rvx,\rvw_0}\left| f_i^{(\ell)}\right|^2$ is uniformly  bounded from $0$ for any $\ell\in\{0,1,...,L-1\}$, $i\in [d_\ell]$.

Specifically, we will prove by induction that $\forall$ $\ell\in\{0,1,...,L-1\}$, $\forall$ $i\in [d_\ell]$,
\begin{align*}
    \mathbb{E}_\rvx\mathbb{E}_{\rvw_0}[| f^{(\ell)}_i|^2] \geq \min\left\{1, \min_{1\leq j\leq \ell}C_\sigma^{\sum_{\ell'=0}^{j-1}r^{\ell'}}\right\}.
\end{align*}

When $\ell=0$, $\mathbb{E}_\vx \left[|x_i|^2\right] = 1$ for all $i\in[d_0]$ by Assumption~\ref{assump:gaussian_input}.

Suppose for all $\ell \leq q-1$, $\mathbb{E}_\rvx\mathbb{E}_{\rvw_0}[| f^{(\ell)}_i|^2] \geq \min\left(1, \min_{1\leq j\leq q}C_\sigma^{\sum_{\ell'=0}^{j-1}r^{\ell'}}\right)$. When $\ell = q$,

\begin{align*}
    \mathbb{E}_\rvx\mathbb{E}_{\rvw_0}[| f^{(q)}_i|^2] = \mathbb{E}_{\rvw_0}\left[\left|\sigma_i^{(q)}\left(\frac{1}{\sqrt{m_i^{(\ell)}}}(\rvw_i^{(q)})^T f_{\S_i^{(q)}}\right)\right|^2\right]= \mathbb{E}_\rvx\mathbb{E}_{\rvw_0} \mathbb{E}_{z\sim\mathcal{N}(0,1)}\left[\left|\sigma_i^{(q)}\left(\frac{\| f_{\S_i^{(q)}}\|}{\sqrt{m_i^{(q)}}} z\right)\right|^2\right].
\end{align*}

By  Assumption~\ref{assump:homo:act},

\begin{align*}
    \mathbb{E}_\rvx\mathbb{E}_{\rvw_0} \mathbb{E}_{z\sim\mathcal{N}(0,1)}\left[\left|\sigma_i^{(q)}\left(\frac{\| f_{\S_i^{(q)}}\|}{\sqrt{m_i^{(q)}}} z\right)\right|^2\right]&= \mathbb{E}_{z\sim \mathcal{N}(0,1)}[|\sigma_i^{(q)}(z)|^2]\mathbb{E}_\rvx\mathbb{E}_{\rvw_0} \left[\left(\frac{\| f_{\S_i^{(q)}}\|^2}{m_i^{(q)}} \right)^r\right]\\
    &\geq C_\sigma\mathbb{E}_\rvx\mathbb{E}_{\rvw_0} \left[\left(\frac{\| f_{\S_i^{(q)}}\|^2}{m_i^{(q)}} \right)^r\right] 
 \end{align*}   
 
 We use Jensen's inequality,

  \begin{align*}  
     C_\sigma\mathbb{E}_\rvx\mathbb{E}_{\rvw_0} \left[\left(\frac{\| f_{\S_i^{(q)}}\|^2}{m_i^{(q)}} \right)^r\right] &\geq C_\sigma \left(\frac{\mathbb{E}_\rvx\mathbb{E}_{\rvw_0}\left[\| f_{\S_i^{(q)}}\|^2\right]}{m_i^{(q)}} \right)^r
\end{align*}

 Then according to inductive assumption, we have
   \begin{align*}  
    C_\sigma \left(\frac{\mathbb{E}_\rvx\mathbb{E}_{\rvw_0}\left[\| f_{\S_i^{(q)}}\|^2\right]}{m_i^{(q)}} \right)^r &\geq C_\sigma \left(\min\left(1, \min_{1\leq j\leq q}C_\sigma^{\sum_{\ell'=0}^{j-1}r^{\ell'}}\right)\right)^r\\
    &\geq \min_{1\leq j\leq q+1}C_\sigma^{\sum_{\ell'=0}^{j-1}r^{\ell'}}.
\end{align*}
Hence for all $l\leq q$, $\mathbb{E}_\rvx\mathbb{E}_{\rvw_0}[| f^{(\ell)}_i|^2] \geq \min\left(1, \min_{1\leq j\leq q+1}C_\sigma^{\sum_{\ell'=0}^{j-1}r^{\ell'}}\right)$, which finishes the inductive step hence the proof.

Therefore,
\begin{align*}
   \mathbb{E}_{\rvx,\rvw_0} \left[\|\nabla_{\rvw^{(\ell)}}f_k(\rvw_0)\|\right] &= \mathbb{E}_{\rvx,\rvw_0}\left[\frac{1}{\sqrt{m_k^{(\ell)}}}\left\| f_{\S_k^{(\ell)}}\right\|\right]\geq  \sqrt{\min\left(1,\min_{1\leq j\leq L}C_\sigma^{\sum_{\ell'=0}^{j-1}r^{\ell'}}\right)} = \Omega(1).
\end{align*}

\section{Proof of Lemma~\ref{lemma:2_F_norm}}\label{proof:2_F_norm}

We prove the result by induction. 

For the base case when $ \ell = \ell'+1$,
\begin{align*}
   \left\|\frac{\partial f_{S_j^{(\ell)}}}{\partial \rvw^{(\ell-1)}}\right\| &=  \left\|\frac{\partial f^{(\ell-1)}}{\partial \rvw^{(\ell-1)}}\frac{\partial  f_{S^{(\ell)}_j}}{\partial  f^{(\ell-1)}}\right\|\\
   &\leq\max_{i: f_i^{(\ell-1)}\in{\F_{\S_j^{(\ell)}}}} \frac{1}{\sqrt{m^{(\ell-1)}_i}}\left|\sigma'(\tilde{ f}_i^{(\ell-1)})\right|\left\| f_{\S_i^{(\ell-1)}}\right\|\\
   &\leq  \max_{i: f_i^{(\ell-1)}\in{\F_{\S_j^{(\ell)}}}} \frac{\gamma_1}{\sqrt{m^{(\ell-1)}_i}}\left\| f_{\S_i^{(\ell-1)}}\right\|.
\end{align*}
and 
\begin{align*}
      \left\|\frac{\partial f_{S_j^{(\ell)}}}{\partial \rvw^{(\ell-1)}}\right\|_F &= \sqrt{\sum_{i: f_i^{(\ell-1)} \in \F_{\S_j^{(\ell)}}}\left\|\frac{\partial  f^{(\ell-1)}_i}{\partial \rvw^{(\ell-1)}}\right\|^2} \\
      &\leq \sqrt{m_j^{(\ell)}}\max_{i: f_i^{(\ell-1)} \in \F_{\S_j^{(\ell)}}}\frac{1}{\sqrt{m_i^{(\ell-1)}}}\left|\sigma'(\tilde{ f}_i^{(\ell-1)})\right|\left\| f_{\S_i^{(\ell-1)}}\right\| \\
      &\leq \sqrt{m_j^{(\ell)}}\max_{i: f_i^{(\ell-1)}\in{\F_{\S_j^{(\ell)}}}} \frac{\gamma_1}{\sqrt{m^{(\ell-1)}_i}}\left\| f_{\S_i^{(\ell-1)}}\right\|.
\end{align*}
By Lemma~\ref{lemma:alpha_2_norm},  with probability at least $1-m_i^{(\ell-1)}\exp(-C_{\ell-1}^\P \log^2{m})$,
    $\left\| f_{\S_i^{(\ell-1)}}\right\| = O\left((\log m +R)^{\ell-2}\sqrt{m_i^{(\ell-1)}}\right)= \tilde{O}\left(R^{\ell-2}\sqrt{m_i^{(\ell-1)}}\right)$.

For the maximum norm $\max_i\left\| f_{\S_i^{(\ell-1)}}\right\|/\sqrt{m_i^{(\ell-1)}}$, we apply union bound over the indices $i$ such that  $f_i^{(\ell-1)}\in{\F_{\S_j^{(\ell)}}}$, the cardinality of which is at most $\left| \F_{\S_j^{(\ell)}}\right| = m_j^{(\ell)}$. Hence with probability at least $1-m_i^{(\ell-1)} m_j^{(\ell)}\exp(-C_{\ell-1}^\P \log^2{m})$, 
\begin{align*}
    \max_i\left\| f_{\S_i^{(\ell-1)}}\right\|/\sqrt{m_i^{(\ell-1)}} = O\left((\log m +R)^{\ell-2}\right)=\tilde{O}(R^{\ell-2}).
\end{align*}

Since $m_i^{(\ell-1)} \leq \overline{m}_{\ell-1}$ and $m_j^{(\ell)} \leq \overline{m}_{\ell}$ where  $ \overline{m}_{\ell-1}, \overline{m}_{\ell}$ are polynomial in $m$, we can find a constant $C_{\ell,\ell-1}^ f>0$ such that $\exp(-C_{\ell,\ell-1}^ f \log^{2}{m}) \geq \exp(-C_{\ell-1}^\P \log^{2}{m}) \cdot \exp(\log(\overline{m}_{\ell-1}\cdot \overline{m}_{\ell}))$. As a result, with probability at least $1-\exp(-C_{\ell,\ell-1}^ f \log^2{m})$,
\begin{align*}
   \left\|\frac{\partial f_{S_j^{(\ell)}}}{\partial \rvw^{(\ell-1)}}\right\| &= O\left((\log m +R)^{\ell-1}\right) = \tilde{O}(R^{\ell-1}),\\
   \left\|\frac{\partial f_{S_j^{(\ell)}}}{\partial \rvw^{(\ell-1)}}\right\|_F &= {O}\left(\sqrt{m_j^{(\ell)}}(\log m +R)^{\ell-1}\right)= \tilde{O}\left(\sqrt{m_j^{(\ell)}}R^{\ell-1}\right).
\end{align*}

Supposing $\ell \leq k$, Eq.~(\ref{eq:2norm}) and (\ref{eq:Fnorm}) hold with probability at least $1-\exp(-C_{k,\ell'}^ f \log^2{m})$. 

For $\ell=k+1$, since elements of $ f_{\S_j^{(k+1)}}$ are from $\P^{(k)}$ where only  $ f^{(\ell')},..., f^{(k)}$ possibly depend on $\rvw^{(\ell')}$, we have
\begin{align}\label{eq:s_j_l'}
    \frac{\partial f_{\S_j^{(k+1)}}}{\partial \rvw^{(\ell')}}  &= \sum_{q = \ell'+1}^{k}\sum_{i: f_i^{(q)}\in\F_{\S_j^{(k+1)}}}\frac{\partial f_{\S_i^{(q)}}}{\partial \rvw^{(\ell')}}\frac{\partial f_i^{(q)}}{\partial  f_{\S_i^{(q)}}} \frac{\partial  f_{\S_j^{(k+1)}}}{\partial  f_i^{(q)}}.
\end{align}
With simple computation, we know that for any $i$ s.t. $ f_i^{(q)}\in\F_{\S_j^{(k+1)}}$:
\begin{align*}
  \frac{\partial f_i^{(q)}}{\partial  f_{\S_i^{(q)}}} \frac{\partial  f_{\S_j^{(k+1)}}}{\partial  f_i^{(q)}} = \frac{1}{\sqrt{m_i^{(q)}}}\sigma'(\tilde{ f}^{(q)}_i)  \rvw_i^{(q)} \frac{\partial  f_{\S_j^{(k+1)}}}{\partial  f_i^{(q)}} ,
\end{align*}
where $\frac{\partial  f_{\S_j^{(k+1)}}}{\partial  f_i^{(q)}}$ is a mask matrix defined in Eq.~(\ref{eq:mask}).

Supposing $\partial f_{\S_i^{(q)}}/{\partial \rvw^{(\ell')}}$, $i\in[d_q]$ in Eq.~(\ref{eq:s_j_l'}) is fixed, for each $q$, we apply Lemma~\ref{lemma:2norm} to bound the spectral norm. Choosing $t= \sqrt{m_j^{(k+1)}}\log{m}$, with probability at least $1-2\exp(-m_j^{(k+1)}\log^2{m})$, for some absolute constant $C>0$,
\begin{align}
   &\left\| \sum_{i: f_i^{(q)}\in\F_{\S_j^{(k+1)}}}\frac{\partial f_{\S_i^{(q)}}}{\partial \rvw^{(\ell')}}\frac{\partial f_i^{(q)}}{\partial  f_{\S_i^{(q-1)}}} \frac{\partial  f_{\S_j^{(k+1)}}}{\partial  f_i^{(q)}}\right\| \nonumber\\
   &\leq {C \gamma_1}\left(\max_i \frac{1}{\sqrt{m_i^{(q)}}} \left\| \frac{\partial f_{\S_i^{(q)}}}{\partial \rvw^{(\ell')}}\right\|\left(\sqrt{m_j^{(k+1)}} + \sqrt{m_j^{(k+1)}}\log{m} + R\right) + \max_i \frac{1}{\sqrt{m_i^{(q)}}}  \left\| \frac{\partial f_{S^{(q)}_i}}{\partial \rvw^{(\ell')}}\right\|_F\right) \label{eq:2_norm_intm}.
\end{align}

To bound the Frobenious norm of Eq.~(\ref{eq:s_j_l'}) for each $q$,  we apply Lemma~\ref{lemma:vector_con} and choose $t = \left\| {\partial f_{S_i^{(q)}}}/{\partial \rvw^{(\ell')}}\right\|\log{m}$.  By union bound over indices $i$ such that $f_i^{(q)}\in \F_{\S_j^{(k+1)}}$, then with probability at least $1-2 m_j^{(k+1)} \exp(-c'\log^2{m})$, where  $c'>0$ is a constant, we have
\begin{align}
     &\left\| \sum_{i: f_i^{(q)}\in\F_{\S_j^{(k+1)}}}\frac{\partial f_{\S_i^{(q)}}}{\partial \rvw^{(\ell')}}\frac{\partial f_i^{(q)}}{\partial  f_{\S_i^{(q)}}} \frac{\partial  f_{\S_j^{(k+1)}}}{\partial  f_i^{(q)}}\right\|_F \nonumber \\
     &= \sqrt{\sum_{i: f_i^{(q)}\in\F_{\S_j^{(k+1)}}} \left\| \frac{\partial f_{\S_i^{(q)}}}{\partial \rvw^{(\ell')}}\frac{\partial f_i^{(q)}}{\partial  f_{\S_i^{(q)}}}\right\|^2} \nonumber \\
     &\leq \sqrt{m_j^{(k+1)}} \max_i \left\|\frac{\partial f_{\S^{(q)}_i}}{\partial \rvw^{(\ell')}}\frac{1}{\sqrt{m_i^{(q)}}}\left(\rvw^{(q)}_i\right)^T \sigma'(\tilde{ f}^{(q)}_i)\right\| \nonumber \\
     &\leq \gamma_1\sqrt{m_j^{(k+1)}} \max_{i}\frac{1}{\sqrt{m_i^{(q)}}} \left( \left\| \frac{\partial f_{\S^{(q)}_i}}{\partial \rvw^{(\ell')}}\right\|(\log{m}+R) +  \left\| \frac{\partial f_{\S^{(q)}_i}}{\partial \rvw^{(\ell')}}\right\|_F \right)\label{eq:F_norm_intm}.
\end{align}

To bound the maximum of $\left\| {\partial f_{\S^{(q)}_i}}/{\partial \rvw^{(\ell')}}\right\| / \sqrt{m_i^{(q)}}$ and $\left\| {\partial f_{\S^{(q)}_i}}/{\partial \rvw^{(\ell')}}\right\|_F / \sqrt{m_i^{(q)}}$ that appear in Eq.~(\ref{eq:2_norm_intm}) and (\ref{eq:F_norm_intm}), with the induction hypothesis, we apply union bound over indices $i$ such that $ f_i^{(q)} \in \F_{\S_j^{(k+1)}}$.  Therefore, with probability at least $1-m_j^{(k+1)}\exp(-C_{q,\ell'}^ f \log^2{m})$,
\begin{align*}
       &\max_i \frac{1}{\sqrt{m_i^{(q)}}}\left\|\frac{\partial f_{\S^{(q)}_i}}{\partial \rvw^{(\ell')}}\right\|= {O}\left(\max_{\ell'+1\leq p\leq q}\frac{1}{\sqrt{\underline{m}_{\,p}}}(\log m + R)^{\ell'}\right)  = \tilde{O}\left(\max_{\ell'+1\leq p\leq q}\frac{R^{\ell'}}{\sqrt{\underline{m}_{\,p}}}\right), \\
   &\max_i \frac{1}{\sqrt{m_i^{(q)}}}\left\|\frac{\partial f_{\S^{(q)}_i}}{\partial \rvw^{(\ell')}}\right\|_F={O}\left((\log m + R)^{q-1}\right)= \tilde{O}\left(R^{q-1}\right).
\end{align*}

Putting them in Eq.~(\ref{eq:2_norm_intm}) and (\ref{eq:F_norm_intm}), we have
\begin{align*}
    &\left\| \sum_{i: f_i^{(q)}\in\F_{\S_j^{(k+1)}}}\frac{\partial f_{\S_i^{(q)}}}{\partial \rvw^{(\ell')}}\frac{\partial f_i^{(q)}}{\partial  f_{\S_i^{(q)}}} \frac{\partial  f_{\S_j^{(k+1)}}}{\partial  f_i^{(q)}}\right\|= \tilde{O}\left(\max \left(\left(\max_{\ell'+1\leq p\leq q}\frac{\sqrt{m_j^{(k+1)}}}{\sqrt{\underline{m_{p}}}}\right),1\right)\right),\\
    &\left\| \sum_{i: f_i^{(q)}\in\F_{\S_j^{(k+1)}}}\frac{\partial f_{\S_i^{(q)}}}{\partial \rvw^{(\ell')}}\frac{\partial f_i^{(q)}}{\partial  f_{\S_i^{(q)}}} \frac{\partial  f_{\S_j^{(k+1)}}}{\partial  f_i^{(q)}}\right\|_F= \tilde{O}\left(\sqrt{m_j^{(k+1)}}\right),
\end{align*}
with probability at least
   $ 1-2\exp(-m_j^{(k+1)}\log^2{m})-m_j^{(k+1)}\exp(-C_{q,\ell'}^ f \log^2{m})
    -2m_j^{(k+1)} \exp(-c'\log^2{m})$.

As the current result is for fixed $q$, applying the union bound over indices $q \in \{\ell'+1,...,k\}$, we have with probability at least
$1-2(k-\ell')\exp(-m_j^{(k+1)})-\sum_q m_j^{(k+1)}\exp(-C_{q,\ell'}^ f \log^2{m}) - 2m_j^{(k+1)} \exp(-c'\log^2{m})$,
\begin{align*}
&\left\|\sum_{q = \ell'+1}^{k}\sum_{i: f_i^{(q)}\in\F_{\S_j^{(k+1)}}}\frac{\partial f_{\S_i^{(q)}}}{\partial \rvw^{(\ell')}}\frac{\partial f_i^{(q)}}{\partial  f_{\S_i^{(q)}}} \frac{\partial  f_{\S_j^{(k+1)}}}{\partial  f_i^{(q)}} \right\| =    {O}\left(\max \left(\max_{\ell'+1\leq p\leq k}\frac{\sqrt{m_j^{(k+1)}}}{\sqrt{\underline{m}_{\,p}}},1\right) (\log m+R)^{\ell'}\right)\\
&\qquad \qquad \qquad \qquad \qquad\qquad\qquad\qquad\qquad~~~~= \tilde{O}\left(\max_{\ell'+1\leq p\leq k+1}\frac{\sqrt{m_j^{(k+1)}}}{\sqrt{\underline{m}_{\,p}}}R^{\ell'}\right),\\
&\left\|\sum_{q = \ell'+1}^{k}\sum_{i: f_i^{(q)}\in\F_{\S_j^{(k+1)}}}\frac{\partial f_{\S_i^{(q)}}}{\partial \rvw^{(\ell')}}\frac{\partial f_i^{(q)}}{\partial  f_{\S_i^{(q)}}} \frac{\partial  f_{\S_j^{(k+1)}}}{\partial  f_i^{(q)}} \right\|_F ={O}\left(\sqrt{m_j^{(k+1)}}(\log m + R)^k\right)= \tilde{O}\left(\sqrt{m_j^{(k+1)}}R^k\right).
\end{align*}
    
Since $m_j^{(k+1)}$ is upper bounded by $\overline{m}_{k+1}$ which is polynomial in $m$, we can find a constant $C_{k+1,\ell'}^ f >0$ such that for each $j$, the result holds with probability at least $1-\exp(-C_{k+1,\ell'}^ f \log^2{m})$ for $\ell\leq k+1$. Then we finish the inductive step which completes the proof.

\section{Proof of Lemma~\ref{lemma:2diff}}\label{proof:2diff}

Before the proof, by Assumption~\ref{assump:md}, we have the following proposition which is critical in the tail bound of the norm of the matrix Gaussian  series, i.e., Lemma~\ref{lemma:gussian_series}. In the bound,  there will be a dimension factor  which is  the number of parameters (see Eq.~(\ref{eq:tropp})). Note that the number of parameters at each layer can be exponentially large w.r.t. the width $m$. If we naively apply the bound, the bound will be useless. However, each neuron in fact only depends on polynomial in $m$ number of parameters, which is the dimension factor we should use. 
\begin{proposition}\label{prop:m_bound}
Fixed $\ell'\in[L]$,  we denote the maximum number of elements in $\rvw^{(\ell')}$ that $f_i^{(\ell)}$ depends on for all $\ell\in[L], i\in [d_\ell]$  by $m^*_{\ell'}$, which is polynomial in $m$.
\end{proposition}

The proof the proposition can be found in Appendix~\ref{proof:m_bound}.

Now we start the proof of the lemma. In fact, we will prove a more general result which includes the neurons in output layer, i.e. $\ell$-th layer. And we will use the result of Lemma~\ref{lemma:2_F_norm} in the proof. Specifically, we will prove the following lemma:

\begin{lemma}
Given $1\leq \ell_1\leq \ell_2 \leq L$, for any $\ell_2\leq \ell\leq L$, $\rvw \in \mathsf{B}(\rvw_0,R)$,  and $j\in [d_{\ell}]$, we have, with probability at least $1-\exp(-\Omega(\log^2 m))$,
\begin{align}
    \left\|\frac{\partial^2 \tilde{f}_j^{(\ell)}}{\partial \rvw^{(\ell_1)}\partial \rvw^{(\ell_2)}}\right\| = {O}\left(\max_{\ell_1+1\leq p\leq \ell}\frac{1}{\sqrt{\underline{m}_{\,p}}}(\log m + R)^{\ell^2}\right) = \tilde{O}\left(\max_{\ell_1+1\leq p\leq \ell}\frac{R^{\ell^2}}{\sqrt{\underline{m}_{\,p}}}\right).
\end{align}
\end{lemma}

We will prove the results by induction. 

For the base case that $\ell = \ell_2$,
\begin{align*}
    \left\|\frac{\partial^2 \tilde{f}_j^{(\ell_2)}}{\partial \rvw^{(\ell_1)}\partial \rvw^{(\ell_2)}}\right\| = \left\|\frac{1}{\sqrt{m_j^{(\ell_2)}}} \frac{\partial f_{S_j^{(\ell_2)}}}{\partial \rvw^{(\ell_1)}} \right\|.
\end{align*}

By Lemma~\ref{lemma:2_F_norm}, we can find a constant $M_{\ell_1,\ell_2}^{(\ell_2),j}>0$ such that with probability at least $1-\exp\left(-M_{\ell_1,\ell_2}^{(\ell_2),j}\log^2 m\right)$,
\begin{align*}
    \left\|\frac{1}{\sqrt{m_j^{(\ell_2)}}} \frac{\partial f_{S_j^{(\ell_2)}}}{\partial \rvw^{(\ell_1)}} \right\| = {O}\left(\max_{\ell_1+1\leq p\leq \ell_2}\frac{1}{\sqrt{\underline{m}_{\,p}}}(\log m + R)^{\ell_1}\right).
\end{align*}

Suppose for $\ell_2\leq \ell' \leq \ell$, with probability at least $1-\exp\left(-M_{\ell_1,\ell_2}^{(\ell),j}\log^2 m\right)$ for some constant $M_{\ell_1,\ell_2}^{(\ell),j}>0$, Eq.~(\ref{eq:2diff_each}) holds.

When $\ell'=\ell+1$, we have
\begin{align}
    \left\|\frac{\partial^2 \tilde{f}_j^{(\ell+1)}}{\partial \rvw^{(\ell_1)}\partial \rvw^{(\ell_2)}}\right\| &= \left\|\sum_{\ell'=\ell_2}^{\ell}\sum_{i=1}^{d_{\ell'}}\frac{\partial^2 f_i^{(\ell')}}{\partial \rvw^{(\ell_1)}\partial \rvw^{(\ell_2)}}\frac{\partial\tilde{f}^{(\ell+1)}_j}{\partial f_i^{(\ell')}}\right\| \leq \sum_{\ell'=\ell_2}^{\ell}\left\|\sum_{i=1}^{d_{\ell'}}\frac{\partial^2 f_i^{(\ell')}}{\partial \rvw^{(\ell_1)}\partial \rvw^{(\ell_2)}}\frac{\partial\tilde{f}^{(\ell+1)}_j}{\partial f_i^{(\ell')}}\right\|\label{eq:2diff_origin}.
\end{align}

We will bound every term in the above summation. For each term, by definition,
\begin{align}\label{eq:2diff_update}
    \sum_{i=1}^{d_{\ell'}}\frac{\partial^2 f_i^{(\ell')}}{\partial \rvw^{(\ell_1)}\partial \rvw^{(\ell_2)}}\frac{\partial\tilde{f}^{(\ell+1)}_j}{\partial f_i^{(\ell')}} &= \sum_{i=1}^{d_{\ell'}} \frac{\partial^2  \tilde{f}_i^{(\ell')}}{\partial  \rvw^{(\ell_1)}\partial \rvw^{(\ell_2)}}\frac{\partial \tilde{ f}_j^{(\ell+1)}}{\partial  \tilde{f}_i^{(\ell')}}\nonumber\\
    &~~~+ \frac{1}{\sqrt{m_j^{(\ell+1)}}}\sum_{i: f_i^{(\ell')} \in \F_{\S^{(\ell+1)}_j}} (\rvw_j^{(\ell+1)})_{\id^{\ell+1,j}_{\ell',i}} \sigma''(\tilde{ f}^{(\ell')}_i) \frac{\partial \tilde{ f}^{(\ell')}_i}{\partial \rvw^{(\ell_1)}}\left(\frac{\partial \tilde{ f}^{(\ell')}_i}{\partial \rvw^{(\ell_2)}}\right)^T.
\end{align}

For the first term in Eq.~(\ref{eq:2diff_update}), we use Lemma~\ref{lemma:2diff_tilde}. Specifically, we view $U_i = \frac{\partial^2  \tilde{f}_i^{(\ell')}}{\partial  \rvw^{(\ell_1)}\partial \rvw^{(\ell_2)}}$, hence with probability at least $1 - \sum_{k=1}^{\ell-\ell'+1}k(m_{\ell_1}^*+m_{\ell_2}^*)\exp(-\log^2 m/2)$,

\begin{align*}
   \left\| \sum_{i=1}^{d_{\ell'}} \frac{\partial^2  \tilde{f}_i^{(\ell')}}{\partial  \rvw^{(\ell_1)}\partial \rvw^{(\ell_2)}}\frac{\partial \tilde{ f}_j^{(\ell+1)}}{\partial  \tilde{f}_i^{(\ell')}}\right\| ={O}\left({\max_{i: f_i^{(\ell')}\in \F_{\S_j^{(\ell'+1)}}} \left\|\frac{\partial^2  \tilde{f}_i^{(\ell')}}{\partial  \rvw^{(\ell_1)}\partial \rvw^{(\ell_2)}}\right\|} (\log m +R)^{\ell - \ell' +1}\right).
\end{align*}

Here we'd like to note that from Lemma~\ref{lemma:gussian_series}, the tail bound depends on the dimension of $\rvw^{(\ell_1)}$ and $\rvw^{(\ell_2)}$ which are $\sum_{i=1}^{d_{\ell_1}} m_i^{(\ell)}$ and $\sum_{i=1}^{d_{\ell_2}} m_i^{(\ell)}$ respectively. By Assumption~\ref{assump:md}, for any $\ell$, $ m_i^{(\ell)}$ is polynomial in $m$. Therefore, the number of elements in $\rvw^{(\ell)}$ that $ f_j^{(\ell+1)}$ depends on is polynomial in $m$ by Proposition~\ref{prop:m_bound}. And the matrix variance $\tilde{\nu}^{(\ell')}$ in Lemma~\ref{lemma:2diff_tilde} is equivalent to the matrix variance that we only consider the elements in $\rvw^{(\ell_1)}$ and $\rvw^{(\ell_2)}$ that $ f_j^{(\ell+1)}$ depends on, in which case the dimension is polynomial in $m$. Therefore we can use $m_\ell^*$ here. It is the same in the following when we apply matrix Gaussian series tail bound.

Then we apply union bound over indices $i$ such that $f_i^{(\ell')}\in \F_{\S_j^{(\ell'+1)}}$, whose cardinality is at most $m_j^{(\ell+1)}$. By the inductive hypothesis, with probability at least $1 - \sum_{k=1}^{\ell-\ell'+1}k (m_{\ell_1}^*+m_{\ell_2}^*)\exp(-\log^2 m/2) - m_j^{(\ell+1)}\exp\left(-M_{\ell_1,\ell_2}^{(\ell),j}\log^2 m\right)$,
\begin{align*}
    \left\| \sum_{i=1}^{d_{\ell'}} \frac{\partial^2  \tilde{f}_i^{(\ell')}}{\partial  \rvw^{(\ell_1)}\partial \rvw^{(\ell_2)}}\frac{\partial \tilde{ f}_j^{(\ell+1)}}{\partial  \tilde{f}_i^{(\ell')}}\right\| = {O}\left(\max_{\ell_1+1\leq p\leq \ell'}\frac{1}{\sqrt{\underline{m}_{\,p}}} (\log m + R)^{(\ell')^2 + \ell - \ell'+1}\right).
\end{align*}

For the second term in Eq.~(\ref{eq:2diff_update}), we view it as a matrix Gaussian series with respect to $\rvw_j^{(\ell+1)}$. The matrix variance takes the form
\begin{align*}
    &\nu_{\ell_1,\ell_2}^{(\ell'),j} = \frac{1}{m_j^{(\ell+1)}}\max\\
    &\left\{\left\|\sum_{i: f_i^{(\ell')} \in \F_{\S^{(\ell+1)}_j}}\left(\sigma''(\tilde{ f}^{(\ell')}_i)\right)^2 \left\|\frac{\partial \tilde{ f}^{(\ell')}_i}{\partial \rvw^{(\ell_1)}}\right\|^2\frac{\partial \tilde{ f}^{(\ell')}_i}{\partial \rvw^{(\ell_2)}}\left(\frac{\partial \tilde{ f}^{(\ell')}_i}{\partial \rvw^{(\ell_2)}}\right)^T\right\|\right.,\\
    &~~~\left.\left\|\sum_{i: f_i^{(\ell')} \in \F_{\S^{(\ell+1)}_j}}\left(\sigma''(\tilde{ f}^{(\ell')}_i)\right)^2 \left\|\frac{\partial \tilde{ f}^{(\ell')}_i}{\partial \rvw^{(\ell_2)}}\right\|^2\frac{\partial \tilde{ f}^{(\ell')}_i}{\partial \rvw^{(\ell_1)}}\left(\frac{\partial \tilde{ f}^{(\ell')}_i}{\partial \rvw^{(\ell_1)}}\right)^T\right\|
    \right\}.
\end{align*}

We use Lemma~\ref{lemma:matrix_varaince}. By the definition in Eq.~(\ref{eq:matrix_variance}), here $\nu_{\ell_1,\ell_2}^{(\ell'),j} = \max\left\{\mu_{\ell_1,\ell_2}^{(\ell'),j},\mu_{\ell_2,\ell_1}^{(\ell'),j}\right\}$. Hence with probability at least at least $1-\exp\left(-C_{\ell_1,\ell_2}^{(\ell'),j} \log^2{m}\right) - \exp\left(-C_{\ell_2,\ell_1}^{(\ell'),j} \log^2{m}\right)$ for some constant $C_{\ell_1,\ell_2}^{(\ell'),j}, C_{\ell_2,\ell_1}^{(\ell'),j}>0$,  we have
\begin{align*}
    \nu_{\ell_1,\ell_2}^{(\ell'),j} = {O}\left(\max\left(1/m_j^{(\ell+1)},\max_{\ell_1+1\leq p\leq \ell} 1/\underline{m}_{\,p}\right)(\log m + R)^{4\ell' - 2}\right).
\end{align*}

Using Lemma~\ref{lemma:gussian_series} again and choosing $t = \log{m}\sqrt{ \nu_{\ell_1,\ell_2}^{(\ell'),j}}$, we have with probability at least $1-(m_{\ell_2}^* + m_{\ell_1}^*)\exp(-\log^2{m}/2)$,
\begin{align*}
    \left\|\frac{1}{\sqrt{m_j^{(\ell+1)}}}\sum_{i: f_i^{(\ell')} \in \F_{\S^{(\ell+1)}_j}} \left(\rvw_j^{(\ell+1)}\right)_{\id^{\ell+1,j}_{\ell',i}} \sigma''\left(\tilde{ f}^{(\ell')}_i\right) \frac{\partial \tilde{ f}^{(\ell')}_i}{\partial \rvw^{(\ell_1)}}\left(\frac{\partial \tilde{ f}^{(\ell')}_i}{\partial \rvw^{(\ell_2)}}\right)^T\right\|\leq (\log{m}+R)\sqrt{ \nu_{\ell_1,\ell_2}^{(\ell'),j}}.
\end{align*}

Combined the bound on $\nu_{\ell_1,\ell_2}^{(\ell'),j}$, we have with probability at least $1-\exp\left(-C_{\ell_1,\ell_2}^{(\ell'),j} \log^2{m}\right) - \exp\left(-C_{\ell_2,\ell_1}^{(\ell'),j} \log^2{m}\right) - (m_{\ell_2}^* + m_{\ell_1}^*)\exp(-\log^2{m}/2)$,
\begin{align*}
   &~~~~\left\|\frac{1}{\sqrt{m_j^{(\ell+1)}}}\sum_{i: f_i^{(\ell')} \in \F_{\S^{(\ell+1)}_j}} \left(\rvw_j^{(\ell+1)}\right)_{\id^{\ell+1,j}_{\ell',i}} \sigma''\left(\tilde{ f}^{(\ell')}_i\right) \frac{\partial \tilde{ f}^{(\ell')}_i}{\partial \rvw^{(\ell_1)}}\left(\frac{\partial \tilde{ f}^{(\ell')}_i}{\partial \rvw^{(\ell_2)}}\right)^T\right\|\\
   &= {O}\left(\max\left(1/\sqrt{m_j^{(\ell_2+1)}},\max_{\ell_1+1\leq p\leq \ell} 1/\sqrt{\underline{m}_{\,p}}\right)(\log m + R)^{2\ell'}\right)\\
   &= \tilde{O}\left(\max\left(1/\sqrt{m_j^{(\ell_2+1)}},\max_{\ell_1+1\leq p\leq \ell} 1/\sqrt{\underline{m}_{\,p}}\right)R^{2\ell'}\right)\\
   &= \tilde{O}\left(\max_{\ell_1+1\leq p\leq \ell+1} R^{2\ell'}/\sqrt{\underline{m}_{\,p}}\right).
\end{align*}

Now we have bound both terms in Eq.~(\ref{eq:2diff_update}). Combining the bounds, we have with probability at least $1- \sum_{k=1}^{\ell-\ell'+1}k (m_{\ell_1}^*+m_{\ell_2}^*)\exp(-\log^2 m/2) - m_j^{(\ell+1)}\exp\left(-M_{\ell_1,\ell_2}^{(\ell),j}\log^2 m\right) -\exp\left(-C_{\ell_1,\ell_2}^{(\ell'),j} \log^2{m}\right) - \exp\left(-C_{\ell_2,\ell_1}^{(\ell'),j} \log^2{m}\right) - 2(m_{\ell_2}^* + m_{\ell_1}^*)\exp(-\log^2{m}/2) $
\begin{align*}
    \left\|\sum_{i=1}^{d_{\ell'}}\frac{\partial^2 f_i^{(\ell')}}{\partial \rvw^{(\ell_1)}\partial \rvw^{(\ell_2)}}\frac{\partial\tilde{f}^{(\ell+1)}_j}{\partial f_i^{(\ell')}}\right\| = {O}\left(\max_{\ell_1+1\leq p\leq \ell+1} 1/\sqrt{\underline{m}_{\,p}} (\log m + R)^{(\ell'+1)^2 +\ell-\ell'}\right).
\end{align*}

With the above results, to bound Eq.~(\ref{eq:2diff_origin}), we apply the union bound over the layer indices $l' = \ell_2,...,\ell$. We have with probability at least $1- \sum_{\ell'=\ell_2}^l \sum_{k=1}^{\ell-\ell'+1}k (m_{\ell_1}^*+m_{\ell_2}^*)\exp(-\log^2 m/2) -(\ell-\ell_2+1) m_j^{(\ell+1)}\exp\left(-M_{\ell_1,\ell_2}^{(\ell),j}\log^2 m\right) -\sum_{\ell'=\ell_2}^l\exp\left(-C_{\ell_1,\ell_2}^{(\ell'),j} \log^2{m}\right) - \sum_{\ell'=\ell_2}^l\exp\left(-C_{\ell_2,\ell_1}^{(\ell'),j} \log^2{m}\right) - 2(\ell-\ell_2+1)(m_{\ell_2}^* + m_{\ell_1}^*)\exp(-\log^2{m}/2)$
\begin{align*}
    \left\|\frac{\partial^2 \tilde{f}_j^{(\ell+1)}}{\partial \rvw^{(\ell_1)}\partial \rvw^{(\ell_2)}}\right\| &\leq \sum_{\ell'=\ell_2}^{\ell}\left\|\sum_{i=1}^{d_{\ell'}}\frac{\partial^2 f_i^{(\ell')}}{\partial \rvw^{(\ell_1)}\partial \rvw^{(\ell_2)}}\frac{\partial\tilde{f}^{(\ell+1)}_j}{\partial f_i^{(\ell')}}\right\|\\
    &= {O}\left(\max_{\ell_1+1\leq p\leq \ell+1} 1/\sqrt{\underline{m}_{\,p}}(\log m + R)^{(\ell+1)^2}\right)\\
    &= \tilde{O}\left(\max_{\ell_1+1\leq p\leq \ell+1} R^{(\ell+1)^2}/\sqrt{\underline{m}_{\,p}}\right).
\end{align*}
By Proposition~\ref{prop:m_bound}, $m_{\ell_1}^*, m_{\ell_2}^*$ are also polynomial in $m$. Hence, we can find a constant $M_{\ell_1,\ell_2}^{(\ell+1),j}>0$ such that
\begin{align*}
   &\exp\left(-M_{\ell_1,\ell_2}^{(\ell+1),j} \log^2 m\right)\\
   &> \sum_{\ell'=\ell_2}^\ell \sum_{k=1}^{\ell-\ell'+1}k (m_{\ell_1}^*+m_{\ell_2}^*)\exp(-\log^2 m/2) -(\ell-\ell_2+1) m_j^{(\ell+1)}\exp\left(-M_{\ell_1,\ell_2}^{(\ell),j}\log^2 m\right)\\
   &~~~~-\sum_{\ell'=\ell_2}^\ell\exp\left(-C_{\ell_1,\ell_2}^{(\ell'),j} \log^2{m}\right) - \sum_{\ell'=\ell_2}^\ell\exp\left(-C_{\ell_2,\ell_1}^{(\ell'),j} \log^2{m}\right) - 2(\ell-\ell_2+1)(m_{\ell_2}^* + m_{\ell_1}^*)\exp(-\log^2{m}/2) \\
   &~~~~+ \exp\left(-M_{\ell_1,\ell_2}^{(\ell),j}\log^2{m}\right).
\end{align*}

Then Eq.~(\ref{eq:2diff_update}) holds with probability at least $1- \exp\left(-M_{\ell_1,\ell_2}^{(\ell+1),j} \log^2 m\right)$ for any $\ell_2\leq \ell+1\leq L$, $j\in[d_{\ell+1}]$, which finishes the induction step hence completes the proof.

\section{Proof of Proposition~\ref{prop:m_bound}}\label{proof:m_bound}

Fixing $\ell'\in[L]$, for any $\ell\in\{\ell',...,L\}$, $i\in[d_\ell]$, we first show $f_i^{(\ell)}$ depends on polynomial number of elements in $\rvw^{(\ell')}$. We prove the result by induction.

If $\ell =\ell'$, then the number of elements in $\rvw^{(\ell)}$ that $ f_i^{(\ell)}$ depend on is $m_i^{(\ell)}$.

Suppose $\ell' \leq \ell \leq k $ that $ f_i^{(\ell)}$ depends on polynomial number of elements in $\rvw^{(\ell')}$. Then at $\ell = k+1$, we know
\begin{align*}
       f_i^{(k+1)} = \sigma_i^{(k+1)} \left(\frac{1}{\sqrt{m_i^{(k+1)}}}\left\langle\rvw_i^{(k+1)},  f_{\S_i^{(k+1)}}\right\rangle\right).
\end{align*}
As $ f_{\S_i^{(k+1)}}$ contains $m_i^{(k+1)}$ neurons where each one depends on polynomial number of elements in $\rvw^{(\ell')}$ by the induction hypothesis. The composition of two polynomial functions is still polynomial, hence $  f_i^{(k+1)}$ also depends on polynomial number of elements in $\rvw^{(\ell')}$.

The maximum number of elements in $\rvw^{(\ell')}$ that $f_i^{(\ell)}$ depends on among all layers $\ell$ is polynomial since it is the maximum of a finite sequence. By Assumption~\ref{assump:md} that $\sup_{\ell\in \{2,...,L-1\},i\in[d_\ell]} m_i^{(\ell)} =O(m^c)$, it is not hard to see that the maximum among all $i\in [d_\ell]$ is also polynomial.

\section{Useful Lemmas and their proofs}
\begin{lemma}\label{lemma:sum_hessian}
Spectral norm of a matrix $H$ is upper bounded by the sum of the spectral norm of
its blocks.
\end{lemma}
\begin{proof}
\begin{align*}
\| H \| &= \left\|\left(\begin{array}{cccc}
    H^{(1,1)} & 0 & \cdots & 0\\
    0 & 0 & \cdots & 0\\
    \vdots & \vdots & \ddots & \vdots \\
    0 & 0 & \cdots & 0
    \end{array}
    \right) + \left(\begin{array}{cccc}
    0 & H^{(1,2)} & \cdots & 0\\
    0 & 0 & \cdots & 0\\
    \vdots & \vdots & \ddots & \vdots \\
    0 &0& \cdots & 0
    \end{array}
    \right)  + \cdots + \left(\begin{array}{cccc}
    0 & 0 & \cdots & 0\\
    0 & 0 & \cdots & 0\\
    \vdots & \vdots & \ddots & \vdots \\
    0 &0& \cdots &H^{(L,L)}
    \end{array}
    \right)\right\| \\
    &\leq \sum_{\ell_1,\ell_2} \| H^{(\ell_1,\ell_2)}\|.
\end{align*}
\end{proof}

\begin{lemma}\label{lemma:alpha_norm}
For $\ell = 0,1,..,L$, with probability at least  $1 -  \exp(-C_\ell^\P \log^2{m})$ for some constant $C_\ell^\P>0$, the absolute value of  all neurons in $\P^{(\ell)}$ Eq.~(\ref{eq:pool}) is of the order  $\tilde{O}(1)$ in the ball $\mathsf{B}(\rvw_0,R)$.
\end{lemma}
\begin{proof}\label{proof:lemma:alpha_norm}
We prove the result by induction. 

When $\ell=0$, $\P^{(0)} =  f^{(0)} = \{x_1,...,x_{d_0}\}$ therefore for all $i$, $| f_i^{(0)}| \leq C_\vx$ surely by~Assumption~\ref{assump:input}.

Suppose when $\ell=k$, with probability at least $1-\exp(-C^\P_k\log^2{m})$, the absolute value of each neuron in $\P^{(k)}$ is of the order ${O}\left((\log m +R)^k\right)$ where  $C^\P_k>0$ is a constant. Then when $\ell=k+1$, there will be new neurons $f^{(k+1)}$ added to $\P^{(k)}$, where each $ f_i^{(k+1)}$  can be bounded by
\begin{align*}
    | f_i^{(k+1)}| &=\left|\sigma\left( \frac{1}{\sqrt{m_i^{(k+1)}}}\left(\rvw_{i}^{(k+1)}\right)^T f_{\S_i^{(k+1)}}\right) \right|\\
    &\leq \frac{\gamma_1}{\sqrt{m_i^{(k+1)}}}\left(\rvw_{i}^{(k+1)}\right)^T f_{\S_i^{(k+1)}} + \sigma(0).\\
\end{align*}
By the union bound over all the elements in $ f_{\S_i^{(k+1)}}$ which are in  $\P^{(k)}$ and the induction hypothesis, with probability at least $1-{m^{(k+1)}_i}\exp(-C^\P_k\log^2{m})$, 
\begin{align*}
    \| f_{\S_i^{(k+1)}}\| = \sqrt{\sum_{j=1}^{m_i^{(k+1)}}\left(f_{\S_i^{(k+1)}}\right)_j} ={O}\left(\sqrt{m_i^{(k+1)}}(\log m + R)^k\right).
\end{align*}
By Lemma~\ref{lemma:innerproduct}, supposing $ f_{\S_i^{(k+1)}}$ is fixed, choosing $t = \log{m}\left\| f_{\S_i^{(k+1)}}\right\|$, with probability at least $1-2\exp(- \log^2{m}/2)$, in the ball $\mathsf{B}(\rvw_0,R)$,
\begin{align*}
    \left|(\rvw_i^{(k+1)})^T f_{\S_i^{(k+1)}}\right| \leq (\log{m}+R)\left\| f_{\S_i^{(k+1)}}\right\|.
\end{align*}
Combined with the bound on $\| f_{\S_i^{(k+1)}}\|$, with probability at least $1-2\exp(- \log^2{m}/2)-{m_i^{(k+1)}}\exp(-C^\P_k\log^2{m})$,
\begin{align*}
    \left| f_i^{(k+1)}\right| \leq \frac{\gamma_1}{\sqrt{m_i^{(k+1)}}} (\log{m}+R)\left\| f_{\S_i^{(k+1)}}\right\| + \gamma_0 =O\left((\log m + R)^{k+1}\right)= \tilde{O}(R^{k+1}).
\end{align*}
Since ${m_i^{(k+1)}} \leq \overline{m}_{k+1}$ which is polynomial in $m$, we can find a constant $C_{k+1}^\P>0$  such that for all $i$,
\begin{align*}
    \exp(-C_{k+1}^\P \log^2{m}) \geq 2\exp(- \log^2{m}/2)+\exp(-C^\P_k\log^{2 }(m)+\log(\overline{m}_{k+1})) + \exp(-C_{k}^\P \log^2{m}).
\end{align*}
Hence the above results hold with probability $1-  \exp(-C_{k+1}^\P \log^2{m})$, which completes the proof.
\end{proof}

\begin{lemma}\label{lemma:alpha_2_norm}
For $\ell\in[L],i\in[d_\ell]$, with probability at least $1-m_i^{(\ell)}\exp(-C_\ell^\P \log^2{m})$, in the ball $\mathsf{B}(\rvw_0,R)$,
\begin{align*}
    \left\| f_{\S_i^{(\ell)}}\right\| = {O}\left(\sqrt{m_i^{(\ell)}}(\log m + R)^{\ell-1}\right) = \tilde{O}\left(\sqrt{m_i^{(\ell)}}R^{\ell-1}\right)
\end{align*}
\end{lemma}

\begin{proof}
By Lemma~\ref{lemma:alpha_norm}, each neuron is of order $\tilde{O}(1)$. Then we apply union bound over $m_i^{(\ell)}$ neurons and we get the result.
\end{proof}

\begin{lemma}\label{lemma:innerproduct}
Given a fixed vector $\vb \in \mathbb{R}^n$ and a random vector $\rva_0\sim\mathcal{N}(0,I_n)$, for any $\rva$ in the ball $\mathsf{B}(\rva_0,R)$, we have with probability at least $1-2\exp(-t^2/(2\|\vb\|^2))$,
\begin{align}
    |\rva^T\vb| \leq t + \|\vb\|R.
\end{align}
\end{lemma}
\begin{proof}
We can write $\rva^T\vb = (\rva_0 + \Delta \rva)^T\vb = \rva_0^T\vb + \Delta \rva^T\vb$. Since $\rva_0 \sim \mathcal{N}(0,1)$, we have $\rva_0^T\vb \sim\mathcal{N}(0,\|\rvb\|^2)$. By Proposition 2.5.2 in~\cite{vershynin2018high}, for any $t>0$, with probability at least $1-2\exp(-t^2/(2\|\vb\|^2))$,
\begin{align*}
    |\rva_0^T\vb| \leq t.
\end{align*}
Therefore, with the same probability
\begin{align*}
    |\rva^T\vb| \leq |\rva_0^T\vb|  + |\Delta \rva^T\vb|\leq t + \|\vb\|R.
\end{align*}
\end{proof}

\begin{lemma}\label{lemma:2norm_init}
For a random $m\times n$ matrix $W = [B_1 \va_1,B_2 \va_2,...,B_n\va_n]$ where $A = [\va_1,\va_2,...,\va_n]$ is an $N_i\times n$ random matrix whose entries i.i.d. follow $\mathcal{N}(0,1)$ and $B_1, B_2,...,B_n$ is a sequence of $m\times N_i$ non-random matrices, we have for some absolute constant $C>0$, for any $t\geq 0$
\begin{align}
    \|W\| \leq C \left(\max_i\|B_i\|(\sqrt{n}+t) +\max_i\|B_i\|_F\right)
\end{align}
with probability at least $1-2\exp(-t^2)$. 

\end{lemma}
\begin{proof}
We prove the result using an $\epsilon$-net argument. Choosing $\epsilon= 1/4$, by Corollary 4.2.13 in~\cite{vershynin2018high}, we can find an $\epsilon$-net $\mathcal{N}$ of the sphere $S^{n-1}$ with cardinalities $|\mathcal{N}| \leq 9^n$.

By Lemma 4.4.1 in~\cite{vershynin2018high}, $\|W\| \leq 2 \sup_{\vx \in \mathcal{N}}\|W\vx\|$.

Fix $\vx\in\mathcal{N}$, it is nor hard to see that 
\begin{align*}
    W\vx = \sum_{i=1}^n x_i B_i\va_i \sim \mathcal{N}\left(0,\sum_{i=1}^n x_i^2 B_iB_i^T \right),
\end{align*}
which can be viewed as $B'\vz$ where $B'=\sqrt{\sum_{i=1}^n x_i^2 B_iB_i^T}$ and $\vz \sim \mathcal{N}(0,I_m)$. 

 By Theorem 6.3.2 in~\cite{vershynin2018high}, we have
\begin{align*}
    \left\| \|B'\vz\| - \|B' \|_F \right\|_{\psi_2} \leq CK^2 \|B'\|,
\end{align*}
where $K = \max_i \| z_i\|_{\psi_2}$ and $\|\cdot\|_{\psi_2}$ is the sub-guassian norm (see Definition 2.5.6 in~\cite{vershynin2018high}) and $C$ is an absolute constant.

By the definition of sub-gaussian norm, we can use the tail bound. For some positive absolute constant $c$ and for any   $\mu>0$,
\begin{align*}
     \mathbb{P} \left\{ \|B'\vz\| - \|B'\|_F \geq u \right\} \leq 2 \exp(-c u^2 /(K^4 \|B'\|^2)).
\end{align*}

Then we unfix $\vx$ using a union bound. With probability at least $1 - 9^n 2 \exp(-c u^2 /(K^4 \|B'\|^2))$
\begin{align*}
    \sup_{\vx \in \mathcal{N}} \|B'\vz\| - \|B'\|_F \leq \mu.
\end{align*}
Choose $u = CK^2\|B'\|(\sqrt{n} + t)$. If the constant $C$ is chosen sufficiently large, we can let $cu^2/K^4 \geq 3n + t^2$. Thus,
\begin{align*}
     \mathbb{P} \left\{\sup_{\vx \in \mathcal{N}} \|B'\vz\| - \|B'\|_F \geq u \right\} \leq 9^n 2\exp\left(-3n-t^2\right)\leq 2\exp(-t^2).
\end{align*}
Combined with $\|W\| \leq 2 \sup_{\vx \in \mathcal{N}}\|W\vx\|$, we conclude that
\begin{align*}
    \mathbb{P}\left\{ \|W\| \geq 2CK^2 \|B'\|(\sqrt{n} + t) + 2\|B'\|_F\right\} \leq  2\exp(-t^2).
\end{align*}

Noticing that $\|B'\| \leq \max_i\|B_i\|$ and $\|B'\|_F \leq \max_i\|B_i\|_F$, we have
\begin{align*}
    \mathbb{P}\left\{ \|W\| \geq 2CK^2 \max_{i }\|B_i\|(\sqrt{n} + t) + 2\max_{i}\|B_i\|_F\right\}\leq  2\exp(-t^2).
\end{align*}
We absorb $K$ into $C$ as $K$ is a constant. With abuse of notation of $C$ which is absolute, we have
\begin{align*}
    \mathbb{P}\left\{ \|W\| \geq C (\max_{i}\|B_i\|(\sqrt{n}+t) +\max_{i}\|B_i\|_F)\right\}\leq  2\exp(-t^2).
\end{align*}
\end{proof}
\begin{lemma}\label{lemma:2norm}
For a random $m\times n$ matrix $W = [B_1 \va_1,B_2 \va_2,...,B_n\va_n]$ where $A = [\va_1,\va_2,...,\va_n]$ and $B_1, B_2,...,B_n$ is a sequence of $m\times N$ non-random matrices. Here $A = A_0 + \Delta A$ where $A_0$ is an $N\times n$ random matrix whose entries i.i.d. follow $\mathcal{N}(0,1)$ and $\Delta A$ is a fixed matrix with $\|\Delta A\|_F \leq R$ given constant $R>0$.
We have for some absolute constant $C>0$, for any $t\geq 0$
\begin{align}
    \|W\| \leq C \left(\max_i\|B_i\|(\sqrt{n}+R+t) +\max_i\|B_i\|_F\right)
\end{align}
with probability at least $1-2\exp(-t^2)$. 
\end{lemma}
\begin{proof}

Comparing to Lemma~\ref{lemma:2norm_init}, we only need to bound the norm of $\Delta W$:
\begin{align*}
    \Delta W := [B_1\Delta \va_1,B_2\Delta \va_2,,...,B_n\Delta \va_n],
\end{align*}
where $\Delta A = [\Delta \va_1,\Delta \va_2,...,\Delta \va_n]$.

By the definition that $\|A_0\|_F =  \sqrt{\sum_{i=1}^n \|\Delta \va_i\|^2}$, we have
\begin{align*}
    \|\Delta W\| \leq \|\Delta W\|_F = \sqrt{\sum_{i=1}^n \|B_i\Delta \va_i\|^2}\leq \max_i\|B_i\|\|\Delta A\|_F\leq \max_i\|B_i\|R.
\end{align*}
Therefore, for any $t\geq 0$, with probability at least $1-2\exp(-t^2)$,
\begin{align*}
    \|W\| \leq \|W-\Delta W\| + \|\Delta W\|\leq  C \left(\max_i\|B_i\|(\sqrt{n}+R+t) +\max_i\|B_i\|_F\right).
\end{align*}
\end{proof}

\begin{lemma}\label{lemma:vector_con}
Consider a fixed matrix $B\in \mathbb{R}^{m\times n}$ and a random vector $\rva_0 \sim \mathcal{N}(0,I_n)$. For any $\rva \in \mathbb{R}^n$ in the ball $\mathsf{B}(\rva_0,R)$ given constant $R>0$, for any $t>0$, we have with probability at least $1 - 2\exp(-ct^2/\|B\|^2)$,where $c$ is an absolute constant,
\begin{align}
    \|B \rva\| \leq t + \|B\|_F + \|B\|R.
\end{align}
\end{lemma}
\begin{proof}
By Theorem 6.3.2 in \cite{vershynin2018high}, for any $t>0$,
\begin{align*}
    \mathbb{P}\{|\|B\rva_0\| - \|B\|_F| \geq t\} \leq 2\exp(-ct^2/\|B\|^2),
\end{align*}
where $c>0$ is an absolute constant.

Note that $\|B\rva\| \leq \|B\rva_0\| + \|{B}(\rva - \rva_0)\| \leq\|B\rva_0\| + \|B\|R $. With probability at least $1 - 2\exp(-ct^2/\|B\|^2)$, we have
\begin{align*}
    \|B\rva\| \leq t + \|B\|_F + \|B\|R.
\end{align*}

\end{proof}

\begin{lemma}[Matrix Gaussian series]\label{lemma:gussian_series}
For a sequence of fixed matrices $\{B_k\}_{k=1}^n$ with dimension $d_1\times d_2$ and a sequence of independent standard normal variables $\{\gamma_k\}$, we define $Z = \sum_{k=1}^n (\gamma_k+\Delta \gamma_k) B_k$  where $\{\Delta \gamma_k\}_{k=1}^n$ is a fixed sequence with $\sum_{k=1}^n \Delta \gamma_k^2 \leq R^2$ given constant $R>0$. Then we have for any $t\geq 0$, with probability at least $1-(d_1+d_2)\exp(-t^2/(2\nu)),$
\begin{align}
    \|Z\| \leq t+R\nu,
\end{align}
where 
\begin{align}
    \nu = \max \left\{\left\|\sum_k B_kB_k^T\right\|,\left\|\sum_k B_k^T B_k\right\|\right\}.
\end{align}
\end{lemma}

\begin{proof}
By Theorem 4.1.1 in \cite{tropp2015introduction}, for all $t\geq 0$,
\begin{align}\label{eq:tropp}
    \mathbb{P}(\left\|\sum_{k=1}^n \gamma_k B_k\right\|\geq t)\leq (d_1+d_2)\exp\left(\frac{-t^2}{2\nu}\right).
\end{align}
Since 
\begin{align*}
    \left\|Z - \sum_{k=1}^n \gamma_k B_k\right\| &= \left\|\sum_{k=1}^n \Delta \gamma_k B_k\right\|\\
    &\leq \sqrt{\sum_{k=1}^n (\Delta \gamma_k)^2}\sqrt{\left\|\sum_{k=1}^n B_k B_k^T\right\|} \\
    &\leq R\sqrt{\nu}.
\end{align*}
Then for $Z$, we have
\begin{align*}
     \mathbb{P}(\|Z\|\geq t + R\sqrt{\nu})\leq (d_1+d_2)\exp\left(\frac{-t^2}{2\nu}\right).   
\end{align*}
\end{proof}

\begin{lemma}[Bound on matrix variance]\label{lemma:matrix_varaince}
For any $\ell\in[L], \ell_1,\ell_2\in[\ell], j\in[d_{\ell+1}]$, with probability at least $1-\exp(-C_{\ell_1,\ell_2}^{(\ell),j} \log^2{m})$ for some constant $C_{\ell_1,\ell_2}^{(\ell),j}>0$, we have
\begin{align}
    \mu_{\ell_1,\ell_2}^{(\ell),j}&:=\frac{1}{m_j^{(\ell+1)}}\left\|\sum_{i: f_i^{(\ell)} \in \F_{\S^{(\ell+1)}_j}}\left(\sigma''\left(\tilde{ f}^{(\ell)}_i\right)\right)^2 \left\|\frac{\partial \tilde{ f}^{(\ell)}_i}{\partial \rvw^{(\ell_1)}}\right\|^2 \frac{\partial \tilde{ f}^{(\ell)}_i}{\partial \rvw^{(\ell_2)}}\left(\frac{\partial \tilde{ f}^{(\ell)}_i}{\partial \rvw^{(\ell_2)}}\right)^T\right\| \nonumber\\
    &~= {O}\left(\max\left(1/m_j^{(\ell+1)},\max_{\min(\ell_1,\ell_2)+1\leq p\leq \ell} 1/\underline{m}_{\,p}\right)(\log m +R)^{4\ell-2}\right) \nonumber \\
    &~= \tilde{O}\left(\max\left(1/m_j^{(\ell+1)},\max_{\min(\ell_1,\ell_2)+1\leq p\leq \ell} 1/\underline{m}_{\,p}\right)R^{4\ell-2}\right) \label{eq:matrix_variance}.
\end{align}
\end{lemma}

\begin{proof}
Without lose of generality, we assume $\ell_1\leq \ell_2\leq \ell$.

We consider two scenarios, (a) $\ell_1\leq \ell_2= \ell$ and (b) $\ell_1\leq \ell_2< \ell$.

In the scenario (a), we analyze $\ell_1=\ell_2=\ell$ and $\ell_1<\ell_2=\ell$ respectively.

When $\ell_1= \ell_2=\ell$, by definition,
\begin{align*}
    \frac{\partial \tilde{ f}^{(\ell)}_i}{\partial \rvw^{(\ell_1)}} = \frac{1}{\sqrt{m_i^{(\ell)}}}f_{\S_i^{(\ell)}}.
\end{align*}

By Lemma~\ref{lemma:alpha_2_norm}, with probability at least $1 - m_i^{(\ell)}\exp(-C_{\ell}^\P \log^2{m})$, $\left\| f_{\S^{(\ell)}_i}\right\| = O\left(\sqrt{m_i^{(\ell)}}(\log m+R)^{\ell-1}\right)= \tilde{O}\left(\sqrt{m_i^{(\ell)}}\right)$. Applying union bound over the indices $i$ such that $ f_i^{(\ell)} \in  f_{\S^{(\ell+1)}_j}$, the carnality of which is at most $m_j^{(\ell+1)}$, we have with probability at least $1 - m_i^{(\ell)}m_j^{(\ell+1)}\exp(-C_{\ell}^\P \log^2{m})$, 
\begin{align*}
    \max_{i: f_i^{(\ell)} \in \F_{\S^{(\ell+1)}_j}} \left\|\frac{\partial \tilde{ f}^{(\ell)}_i}{\partial \rvw^{(\ell_1)}}\right\| = \max_{i: f_i^{(\ell)} \in \F_{\S^{(\ell+1)}_j}}\frac{\left\| f_{\S^{(\ell)}_i}\right\|}{\sqrt{m_i^{(\ell)}}}= O\left((\log m + R)^{\ell-1}\right)= \tilde{O}(R^{\ell-1}).
\end{align*}

It is not hard to see that 
\begin{align*}
\sum_{i: f_i^{(\ell)} \in \F_{\S^{(\ell+1)}_j}}\frac{\partial \tilde{ f}^{(\ell)}_i}{\partial \rvw^{(\ell)}}\left(\frac{\partial \tilde{ f}^{(\ell)}_i}{\partial \rvw^{(\ell)}}\right)^T    
\end{align*}
is a block diagonal  matrix with $i$-th block in the form $\frac{1}{{m_i^{(\ell)}}} f_{\S^{(\ell)}_i}\left( f_{\S^{(\ell)}_i}\right)^T \cdot \mathbb{I}\left\{ f_i^{(\ell)} \in  f_{\S^{(\ell+1)}_j} \right\}$.

Therefore, $\mu_{\ell,\ell}^{(\ell),j}$ can be bounded by 
\begin{align*}
   \mu_{\ell,\ell}^{(\ell),j} &\leq \frac{1}{m_j^{(\ell+1)}} \gamma_2^2 \left(\max_{i: f_i^{(\ell)} \in \F_{\S^{(\ell+1)}_j}}\frac{\left\| f_{\S^{(\ell)}_i}\right\|}{\sqrt{m_i^{(\ell)}}}\right)^2 \left(\max_{i: f_i^{(\ell)} \in \F_{\S^{(\ell+1)}_j}}\left\|\frac{1}{{m_i^{(\ell)}}} f_{\S^{(\ell)}_i}\left( f_{\S^{(\ell)}_i}\right)^T\right\|\right)\\
   &= {O}\left((\log m + R)^{4\ell-4}/m_j^{(\ell+1)}\right) = \tilde{O}\left(R^{4\ell-4}/m_j^{(\ell+1)}\right),
\end{align*}
with probability at least $1 - 2m_i^{(\ell)}m_j^{(\ell+1)}\exp(-C_{\ell}^\P \log^2{m})$, where we apply the union bound on $\left\| f_{\S^{(\ell)}_i}\right\|$ once again. 

By definition Eq.~(\ref{eq:up_low_m}), $m_i^{(\ell)} \leq \overline{m}_{\ell}$ and $m_j^{(\ell+1)} \leq \overline{m}_{\ell+1}$. By Assumption~\ref{assump:md}, $ \overline{m}_{\ell}, \overline{m}_{\ell+1}$ are polynomial in $m$. If $m$ is large enough, we can find a constant $C_{\ell,\ell}^{(\ell),j}>0$ such that 
\begin{align*}
    \exp(-C_{\ell,\ell}^{(\ell),j} \log^2{m})> 2m_i^{(\ell)}m_j^{(\ell+1)}\exp(-C_{\ell}^\P \log^2{m}),
\end{align*}
thus the bound holds with probability $1-  \exp\left(-C_{\ell,\ell}^{(\ell),j} \log^2{m}\right)$. 

When $\ell_1<\ell_2=\ell$,
By Eq.~(\ref{eq:basic_op}), we compute the derivative:
\begin{align}
    \frac{\partial \tilde{ f}^{(\ell)}_i}{\partial \rvw^{(\ell_1)}} = \frac{1}{\sqrt{m_i^{(\ell)}}}\frac{\partial  f_{\S^{(\ell)}_i}}{\partial \rvw^{(\ell_1)}} \rvw_i^{(\ell)}.\label{eq:matrix_variance_norm_1}
\end{align}
By Lemma~\ref{lemma:2_F_norm}, with probability at least $1-\exp\left(-C_{\ell,\ell_1}^ f \log^2{m}\right)$, $\left\|\partial  f_{\S^{(\ell)}_i}/\partial \rvw^{(\ell_1)}\right\| = \tilde{O}\left(\max_{\ell_1 + 1\leq p\leq \ell}{\sqrt{m_i^{(\ell)}}} /{\sqrt{\underline{m}_{\,p}}}\right)$ and $\left\|\partial  f_{\S^{(\ell)}_i}/\partial \rvw^{(\ell_1)}\right\|_F = \tilde{O}\left(\sqrt{m_i^{(\ell)}}\right)$. We use Lemma~\ref{lemma:vector_con} and choose $t = \log{m}\left\|\partial  f_{\S^{(\ell)}_i}/\partial \rvw^{(\ell_1)}\right\|$, then with probability at least $1-2\exp(-c'\log^2{m}) - \exp\left(-C_{\ell,\ell_1}^ f \log^2{m}\right)$ for some absolute constant $c'>0$,
\begin{align}
    \left\| \frac{\partial \tilde{ f}^{(\ell)}_i}{\partial \rvw^{(\ell_1)}}\right\| &= \frac{1}{\sqrt{m_i^{(\ell)}}}\left\|\frac{\partial  f_{\S^{(\ell)}_i}}{\partial \rvw^{(\ell_1)}} \rvw^{(\ell)}_i\right\|\\
    &\leq \frac{1}{\sqrt{m_i^{(\ell)}}}\left((\log{m}+R)\left\|\frac{\partial  f_{\S^{(\ell)}_i}}{\partial \rvw^{(\ell_1)}}\right\|+\left\|\frac{\partial  f_{\S^{(\ell)}_i}}{\partial \rvw^{(\ell_1)}}\right\|_F\right)\\
    &=O\left((\log m +R)^{\ell}\right) = \tilde{O}(R^\ell) \label{eq:dfdwl1}.
\end{align}

Similar to the case when $\ell_1=\ell_2=\ell$, \begin{align*}
\sum_{i: f_i^{(\ell)} \in \F_{\S^{(\ell+1)}_j}}\frac{\partial \tilde{ f}^{(\ell)}_i}{\partial \rvw^{(\ell)}}\left(\frac{\partial \tilde{ f}^{(\ell)}_i}{\partial \rvw^{(\ell)}}\right)^T    
\end{align*} is a block matrix. 

Therefore,
\begin{align*}
   \mu_{\ell,\ell}^{(\ell),j} &\leq \frac{1}{m_j^{(\ell+1)}} \gamma_2^2 \left(\max_{i: f_i^{(\ell)} \in \F_{\S^{(\ell+1)}_j}} \left\| \frac{\partial \tilde{ f}^{(\ell)}_i}{\partial \rvw^{(\ell_1)}}\right\|\right)^2 \left(\max_{i: f_i^{(\ell)} \in \F_{\S^{(\ell+1)}_j}}\left\|\frac{1}{{m_i^{(\ell)}}} f_{\S^{(\ell)}_i}\left( f_{\S^{(\ell)}_i}\right)^T\right\|\right)\\
   &= {O}\left((\log m + R)^{4\ell-2}/m_j^{(\ell+1)}\right) = \tilde{O}\left(R^{4\ell-2}/m_j^{(\ell+1)}\right),
\end{align*}
with probability at least $1-2m_j^{(\ell+1)}\exp(-c'\log^2{m}) - m_j^{(\ell+1)}\exp\left(-C_{\ell,\ell_1}^ f \log^2{m}\right)-2m_i^{(\ell)}m_j^{(\ell+1)}\exp\left(-C_{\ell}^\P \log^2{m}\right)$ where we apply the union bound over the indices $i$ for the maximum.

Similarly, we can find a constant $C_{\ell_1,\ell}^{(\ell),j} >0$ such that the bound holds with probability $1-  \exp\left(-C_{\ell,\ell}^{(\ell),j} \log^2{m}\right)$. 

For scenario (b) that $\ell_1\leq \ell_2<\ell$, we apply Lemma~\ref{lemma:2norm} to bound  $\mu_{\ell_1,\ell_2}^{(\ell),j}$ . Specifically, we view 
\begin{align}
    &B_i = \frac{1}{\sqrt{m_i^{(\ell)}}}\left|\sigma''(\tilde{ f}^{(\ell)}_i)\right| \left\|\frac{\partial \tilde{ f}^{(\ell)}_i}{\partial \rvw^{(\ell_1)}}\right\|\frac{\partial  f_{\S^{(\ell)}_i}}{\partial \rvw^{(\ell_2)}} \label{eq:examp_B_i},\\
    &\rva_i = \rvw_i^{(\ell)}.\label{eq:examp_a_i}
\end{align} 
Choosing $t = \log{m}$ and supposing $B_i$ is fixed, then with probability at least $1-2\exp(-\log^2{m})$, for some constant $K^{\ell,j}_{\ell_1,\ell_2}>0$,
\begin{align*}
    &\left\|\sum_{i: f_i^{(\ell)} \in  f_{\S^{(\ell+1)}_j}}\left(\sigma''(\tilde{ f}^{(\ell)}_i)\right)^2 \left\|\frac{\partial \tilde{ f}^{(\ell)}_i}{\partial \rvw^{(\ell_1)}}\right\|^2\frac{\partial \tilde{ f}^{(\ell)}_i}{\partial \rvw^{(\ell_2)}}\left(\frac{\partial \tilde{ f}^{(\ell)}_i}{\partial \rvw^{(\ell_2)}}\right)^T\right\| \\
    &\leq (K^{\ell,j}_{\ell_1,\ell_2})^2 \left(\max_i \|B_i\|\left(\sqrt{{m_j^{(\ell+1)}}} + \log{m}+R\right)+\max_i \|B_i\|_F\right)^2 \\
    &\leq {(K^{\ell,j}_{\ell_1,\ell_2})^2\gamma_2^2} \left(\max_i \frac{1}{\sqrt{m_i^{(\ell)}}}\left\|\frac{\partial \tilde{ f}^{(\ell)}_i}{\partial \rvw^{(\ell_1)}}\right\|\left\|\frac{\partial  f_{\S^{(\ell)}_i}}{\partial \rvw^{(\ell_2)}}\right\|\left(\sqrt{m_j^{(\ell+1)}}+\log{m}+R\right) +  \max_i \frac{1}{\sqrt{m_i^{(\ell)}}}\left\|\frac{\partial \tilde{ f}^{(\ell)}_i}{\partial \rvw^{(\ell_1)}}\right\|\left\|\frac{\partial  f_{\S^{(\ell)}_i}}{\partial \rvw^{(\ell_2)}}\right\|_F\right)^2
\end{align*}

By Eq.~(\ref{eq:dfdwl1}), with probability at least $1-2\exp(-c'\log^2{m}) - \exp\left(-C_{\ell,\ell_1}^ f \log^2{m}\right)$ for some absolute constant $c'>0$,
\begin{align}
    \left\| \frac{\partial \tilde{ f}^{(\ell)}_i}{\partial \rvw^{(\ell_1)}}\right\| 
    = \tilde{O}(R^{\ell}).
\end{align}

By Lemma~\ref{lemma:2_F_norm}, with probability at least $1-\exp(-C_{\ell,\ell_2}^ f \log^2{m})$, $\left\|\partial  f_{\S^{(\ell)}_i}/\partial \rvw^{(\ell_2)}\right\| = \tilde{O}\left(\max_{\ell_2 + 1\leq p\leq \ell}{\sqrt{m_i^{(\ell)}}} /{\sqrt{\underline{m}_{\,p}}}\right)$ and $\left\|\partial  f_{\S^{(\ell)}_i}/\partial \rvw^{(\ell_2)}\right\|_F = \tilde{O}\left(\sqrt{m_i^{(\ell)}}\right)$.

Combined them together, with probability at least $1 - m_i^{(\ell)}m_j^{(\ell+1)}\exp\left(-C_{\ell}^\P \log^2{m}\right)-2m_j^{(\ell+1)}\exp(-c'\log^2{m}) - m_j^{(\ell+1)}\exp\left(-C_{\ell,\ell_1}^ f \log^2{m}\right)- m_j^{(\ell+1)}\exp\left(-C_{\ell,\ell_2}^ f \log^2{m}\right)$,

\begin{align*}
   \mu_{\ell_1,\ell_2}^{(\ell),j}&={O}\left(\max\left(1/m_j^{(\ell+1)},\max_{\ell_1+1\leq p\leq \ell} 1/\underline{m}_{\,p}\right)(\log m + R)^{4\ell-2}\right)\\
   &=\tilde{O}\left(\max\left(1/m_j^{(\ell+1)},\max_{\ell_1+1\leq p\leq \ell} 1/\underline{m}_{\,p}\right)R^{4\ell-2}\right).
\end{align*}
Similarly we can find a constant $C_{\ell_1,\ell_2}^{(\ell),j}>0$ such that with probability at least $1-\exp\left(-C_{\ell_1,\ell_2}^{(\ell),j} \log^2{m}\right)$ the above bound holds.

For $\ell_1\geq \ell_2$, we similarly have 
\begin{align*}
   \mu_{\ell_2,\ell_1}^{(\ell),j}= \tilde{O}\left(\max\left(1/m_j^{(\ell+1)},\max_{\ell_2+1\leq p\leq \ell} 1/\underline{m}_{\,p}\right)R^{4\ell-2}\right),
\end{align*}
with probability at least $1-\exp(-C_{\ell_2,\ell_1}^{(\ell),j} \log^2{m})$.
\end{proof}

\begin{lemma}\label{lemma:2diff_tilde}
For any $0<\ell'\leq \ell\leq L-1$, given fixed matrices $U_1,...,U_{d_{\ell'}} \in\mathbb{R}^{u_1\times u_2}$, with probability at least $1 - \sum_{k=1}^{\ell-\ell'+1}k(u_1+u_2)\exp(-\log^2 m/2)$
\begin{align*}
    \sum_{i=1}^{d_{\ell'}} U_i\frac{\partial \tilde{ f}_j^{(\ell+1)}}{\partial  \tilde{f}_i^{(\ell')}}
    &={O}\left({\max_{i: f_i^{(\ell')}\in \F_{\S_j^{(\ell'+1)}}} \|U_i\| (\log m + R)^{\ell - \ell'+1}}\right)\\
    &=\tilde{O}\left({\max_{i: f_i^{(\ell')}\in \F_{\S_j^{(\ell'+1)}}} \|U_i\|}\right)  .
\end{align*}
\end{lemma}
\begin{proof}
We prove the result by induction.

For the base case that $\ell = \ell'$,
\begin{align*}
     \sum_{i=1}^{d_{\ell'}} U_i\frac{\partial \tilde{ f}_j^{(\ell'+1)}}{\partial  \tilde{f}_i^{(\ell')}} &= \sum_{i=1}^{d_{\ell'}} U_i \sigma'\left(\tilde{f}_i^{(\ell')}\right)\frac{\partial f_{\S_j^{(\ell'+1)}}}{\partial f^{(\ell')}_i}\frac{\partial \tilde{ f}_j^{(\ell'+1)}}{ \partial f_{\S_j^{(\ell'+1)}}} \\
    &= \frac{1}{\sqrt{m_j^{(\ell'+1)}}} \sum_{i: f_i^{(\ell')}\in \F_{\S_j^{(\ell'+1)}}} U_i \sigma'\left(\tilde{f}_i^{(\ell')}\right)\left(\rvw_j^{(\ell'+1)}\right)_{\id^{\ell'+1,j}_{\ell',i}}.
\end{align*}

We view the above equation as a matrix Gaussian series with respect to $\rvw_j^{(\ell'+1)}$. Its matrix variance $\nu^{(\ell')}$ can be bounded by
\begin{align*}
    \nu^{(\ell')} &:= \frac{1}{m_j^{(\ell'+1)}}\left\|\sum_{i: f_i^{(\ell')}\in \F_{\S_j^{(\ell'+1)}}} U_i \sigma'\left(\tilde{f}_i^{(\ell')}\right) \right\|^2\\
    &~\leq \max_{i: f_i^{(\ell')}\in \F_{\S_j^{(\ell'+1)}}} \gamma_1^2\|U_i\|^2 .
\end{align*}

Using Lemma~\ref{lemma:gussian_series} and choosing $t = \log m \sqrt{\nu^{(\ell')}}$, we have with probability at least $1 - (u_1+u_2)\exp(-\log^2 m/2)$,
\begin{align*}
      \sum_{i=1}^{d_{\ell'}} U_i\frac{\partial \tilde{ f}_j^{(\ell'+1)}}{\partial  \tilde{f}_i^{(\ell')}} \leq (\log m + R)\sqrt{\nu^{(\ell')}} \leq \max_i (\log m + R)\gamma_1\|U_i\| = {O}((\log m + R){\max_i \|U_i\|}).
\end{align*}

Suppose with probability at least $1 - \sum_{k=1}^{\ell-\ell'+1}k(u_1+u_2)\exp(-\log^2 m/2)$, for all $\ell'\leq k \leq \ell$, 
\begin{align*}
     \sum_{i=1}^{d_{\ell'}} U_i\frac{\partial \tilde{ f}_j^{(k)}}{\partial  \tilde{f}_i^{(\ell')}} ={O}\left({\max_{i: f_i^{(\ell')}\in \F_{\S_j^{(\ell'+1)}}} \|U_i\|}(\log m + R)^{k-\ell'}\right) .
\end{align*}

Then when $k = \ell+1$, we have

\begin{align*}
     \sum_{i=1}^{d_{\ell'}} U_i\frac{\partial \tilde{ f}_j^{(\ell+1)}}{\partial  \tilde{f}_i^{(\ell')}} &=\sum_{r = \ell'}^{\ell}\sum_{i=1}^{d_{\ell'}} U_i \frac{\partial f^{(r)}}{\partial \tilde{f}_i^{(\ell')}}\frac{\partial f_{\S_j^{(\ell+1)}}}{\partial f^{(r)}}\frac{\partial \tilde{ f}_j^{(\ell+1)}}{ \partial f_{\S_j^{(\ell+1)}}} \\
     &=\sum_{r = \ell'}^{\ell} \sum_{s = 1}^{d_r}\sum_{i=1}^{d_{\ell'}} U_i \frac{\partial f_s^{(r)}}{\partial \tilde{f}_i^{(\ell')}} \frac{\partial f_{\S_j^{(\ell+1)}}}{\partial f_s^{(r)}}\frac{\partial \tilde{ f}_j^{(\ell+1)}}{ \partial f_{\S_j^{(\ell+1)}}}\\
     &= \sum_{r = \ell'}^{\ell} \sum_{s = 1}^{d_r}\sum_{i=1}^{d_{\ell'}} U_i \frac{\partial \tilde{f}_s^{(r)}}{\partial \tilde{f}_i^{(\ell')}} \sigma'\left(\tilde{f}_s^{(r)}\right)\frac{\partial f_{\S_j^{(\ell+1)}}}{\partial \tilde{f}_s^{(r)}}\frac{\partial \tilde{ f}_j^{(\ell+1)}}{ \partial f_{\S_j^{(\ell+1)}}}\\
     &= \sum_{r = \ell'}^{\ell} \sum_{s: f_s^{(r)}\in \F_{\S_j^{(\ell+1)}}}\left(\sum_{i=1}^{d_{\ell'}} U_i \frac{\partial \tilde{f}_s^{(r)}}{\partial \tilde{f}_i^{(\ell')}}\right) \sigma'\left(\tilde{f}_s^{(r)}\right) \frac{1}{\sqrt{m_j^{(\ell+1)}}}\left(\rvw_j^{(\ell+1)}\right)_{\id^{\ell+1,j}_{r,s}}
\end{align*}

For each $r \in\{\ell',...,\ell\}$, we view $\sum_{s: f_s^{(r)}\in \F_{\S_j^{(\ell+1)}}}\left(\sum_{i=1}^{d_{\ell'}} U_i \frac{\partial \tilde{f}_s^{(r)}}{\partial \tilde{f}_i^{(\ell')}}\right) \sigma'\left(\tilde{f}_s^{(r)}\right) \frac{1}{\sqrt{m_j^{(\ell+1)}}}\left(\rvw_j^{(\ell+1)}\right)_{\id^{\ell+1,j}_{r,s}}$ as a matrix Gaussian series with respect to $\rvw_j^{(\ell+1)}$.

By the inductive hypothesis, for all $r$, its matrix variance can be bounded by 
\begin{align*}
    \nu^{(r)} &:= \frac{1}{m_j^{(\ell+1)}}\left\|\sum_{s: f_s^{(r)}\in \F_{\S_j^{(\ell+1)}}}\left(\sum_{i=1}^{d_{\ell'}} U_i \frac{\partial \tilde{f}_s^{(r)}}{\partial \tilde{f}_i^{(\ell')}}\right) \sigma'\left(\tilde{f}_s^{(r)}\right) \right\|^2\\
    &~= {O}\left({\max_{i: f_i^{(\ell')}\in \F_{\S_j^{(\ell'+1)}}} \|U_i\|^2} (\log m + R)^{2r-2\ell'}\right).
\end{align*}

Then we use Lemma~\ref{lemma:gussian_series} and choose $t = \log m \sqrt{\nu^{(r)}}$.  With probability at least $1 - (u_1+u_2)\exp(-\log^2 m/2)$,
\begin{align*}
      &~~~~\left\|\sum_{s: f_s^{(r)}\in \F_{\S_j^{(\ell+1)}}}\left(\sum_{i=1}^{d_{\ell'}} U_i \frac{\partial \tilde{f}_s^{(r)}}{\partial \tilde{f}_i^{(\ell')}}\right) \sigma'\left(\tilde{f}_s^{(r)}\right) \frac{1}{\sqrt{m_j^{(\ell+1)}}}\left(\rvw_j^{(\ell+1)}\right)_{\id^{\ell+1,j}_{r,s}}\right\|\\
      &\leq (\log m + R)\sqrt{\nu^{(r)}} \\
      &\leq \max_{i: f_i^{(\ell')}\in \F_{\S_j^{(\ell'+1)}}} (\log m + R)\gamma_1\|U_i\|\\
      &= {O}\left({\max_{i: f_i^{(\ell')}\in \F_{\S_j^{(\ell'+1)}}} \|U_i\|}(\log m + R)^{r-\ell'+1}\right).
\end{align*}

We apply union bound over indices $r = \ell',...,\ell$ and add the probability from the induction hypothesis. With probability at least $1 - \sum_{k=1}^{\ell-\ell'+1}k(u_1+u_2)\exp(-\log^2 m/2)$,
\begin{align*}
    \left\| \sum_{i=1}^{d_{\ell'}} U_i\frac{\partial \tilde{ f}_j^{(\ell+1)}}{\partial  \tilde{f}_i^{(\ell')}}\right\| &\leq \sum_{r = \ell'+1}^{\ell} \left\|\sum_{s: f_s^{(r)}\in \F_{\S_j^{(\ell+1)}}}\left(\sum_{i=1}^{d_{\ell'}} U_i \frac{\partial \tilde{f}_s^{(r)}}{\partial \tilde{f}_i^{(\ell')}}\right) \sigma'\left(\tilde{f}_s^{(r)}\right) \frac{1}{\sqrt{m_j^{(\ell+1)}}}\left(\rvw_j^{(\ell+1)}\right)_{\id^{\ell+1,j}_{r,s}} \right\|\\
    &= {O}\left({\max_{i: f_i^{(\ell')}\in \F_{\S_j^{(\ell'+1)}}} \|U_i\|}(\log m + R)^{\ell - \ell'+1}\right)\\
    &= \tilde{O}\left({\max_{i: f_i^{(\ell')}\in \F_{\S_j^{(\ell'+1)}}} \|U_i\|}R^{\ell-\ell'+1}\right).
\end{align*}

Then we finish the induction step which completes the proof.
\end{proof}

\end{document}